 \documentclass[tablecaption=bottom,wcp,colorlinks,linktoc=all,usenames,dvipsnames,inline]{jmlr} % W&CP article

\usepackage{booktabs}
\usepackage[load-configurations=version-1]{siunitx} % newer version
\usepackage[T1]{fontenc}

\usepackage{amssymb}

\definecolor{shadecolor}{gray}{0.9}

\usepackage{graphicx}
\usepackage[labelfont=bf]{caption}
\usepackage{booktabs, array, multirow}
\usepackage{wrapfig}

\usepackage{natbib}
\usepackage[all]{hypcap}
\hypersetup{citecolor=MidnightBlue}
\hypersetup{linkcolor=MidnightBlue}
\hypersetup{urlcolor=MidnightBlue}
\usepackage[nameinlink]{cleveref}
\creflabelformat{equation}{#2\textup{#1}#3}  % <- remove parenthesis from equations

\usepackage
[acronym,nowarn,section,nogroupskip,nonumberlist]{glossaries}
\glsdisablehyper{}

\input{maths-preamble}
\newacronym[longplural={neural processes}]{np}{NP}{neural process}
\newacronym[longplural={transformer neural processes}]{tnp}{TNP}{transformer neural process}
\newacronym{te-tnp}{TE-TNP}{translation-equivariant TNP}
\newacronym{pt-tnp}{PT-TNP}{pseudo-token TNP}
\newacronym{te-pt-tnp}{TE-PT-TNP}{translation-equivariant PT-TNP}
\newacronym{cnp}{CNP}{conditional NP}
\newacronym{anp}{ANP}{attentive NP}
\newacronym{tnp-d}{TNP-D}{diagonal TNP}
\newacronym{tnp-ar}{TNP-AR}{autoregressive TNP}
\newacronym{tnp-nd}{TNP-ND}{non-diagonal TNP}
\newacronym{eq-tnp}{EQTNP}{efficient-query TNP}
\newacronym{lbanp}{LBANP}{latent-bottlenecked ANP}
\newacronym{ist}{IST}{induced set transformer}
\newacronym{convcnp}{ConvCNP}{convolutional conditional NP}
\newacronym{rcnp}{RCNP}{relational CNP}
\newacronym{mhsa}{MHSA}{multi-head self attention}
\newacronym{mhca}{MHCA}{multi-head cross attention}
\newacronym{te-mhsa}{TE-MHSA}{translation-equivariant multi-head self attention}
\newacronym{te-mhca}{TE-MHCA}{translation-equivariant multi-head cross attention}
\newacronym{model}{ICICL-TNP}{in-context in-context learning TNP}

\jmlrproceedings{AABI 2024}{Proceedings of the 6th Symposium on Advances in Approximate Bayesian Inference, 2024}

\title[In-Context In-Context Learning]{In-Context In-Context Learning with Transformer Neural Processes}

 \author{\Name{Matthew Ashman} \Email{mca39@cam.ac.uk}\\
 \addr University of Cambridge
 \AND
 \Name{Cristiana Diaconu} \Email{cdd43@cam.ac.uk}\\
 \addr University of Cambridge
 \AND
 \Name{Adrian Weller} \Email{aw665@cam.ac.uk}\\
 \addr University of Cambridge \\
 \addr The Alan Turing Institute
 \AND
 \Name{Richard E.\ Turner} \Email{ret26@cam.ac.uk}\\
 \addr University of Cambridge \\
 \addr Microsoft Research AI for Science
 }

\begin{document}

\maketitle

\begin{abstract}
Neural processes (NPs) are a powerful family of meta-learning models that seek to approximate the posterior predictive map of the ground-truth stochastic process from which each dataset in a meta-dataset is sampled. %Yet, in many cases, the \emph{marginal} stochastic process can be deconstructed into a mixture of stochastic processes, each of which having generated a number of datasets contained in the meta-dataset. In such cases, it is beneficial to condition on all the datasets that are drawn from the same stochastic process when making predictive inferences. We describe the paradigm of learning predictive maps of this form as \emph{in-context in-context learning}. 
%There are many situations in which we have access to a dataset we want to make predictions on, as well as a set of other related datasets (e.g. they share characteristics with the dataset of interest). We hypothesise that in this case, additionally conditioning on those related datasets can improve the predictive performance of NPs.
There are many cases in which practitioners, besides having access to the dataset of interest, may also have access to other datasets that share similarities with it. In this case, integrating these datasets into the NP can improve predictions. We equip NPs with this functionality and describe this paradigm as \emph{in-context in-context learning}. Standard NP architectures, such as the convolutional conditional NP (ConvCNP) or the family of transformer neural processes (TNPs), are not capable of in-context in-context learning, as they are only able to condition on a single dataset. We address this shortcoming by developing the in-context in-context learning pseudo-token TNP (ICICL-TNP). The ICICL-TNP builds on the family of PT-TNPs, which utilise pseudo-token-based transformer architectures to sidestep the quadratic computational complexity associated with regular transformer architectures. Importantly, the ICICL-TNP is capable of conditioning on both sets of datapoints and sets of datasets, enabling it to perform in-context in-context learning. We demonstrate the importance of in-context in-context learning and the effectiveness of the ICICL-TNP in a number of experiments.
\end{abstract}

% Keywords may be removed
%\begin{keywords}
%List of keywords
%\end{keywords}

\section{Introduction}
\label{sec:intro}
\Glspl{np} are a broad family of meta-learning models which learn the mapping from sets of observed datapoints to predictive distributions \citep{foong2020meta}. 
They enjoy a number of attractive properties, most notably their ability to effectively model data from different modalities and drawn from complex stochastic process priors.
This flexibility makes \glspl{np} a popular choice for a wide variety of problem domains, including spatio-temporal modelling, healthcare, and few-shot learning \citep{jha2022neural}.
A key consideration of \gls{np} architectures is the ability to handle an arbitrary number of observed datapoints in an exchangeable fashion. This is achieved through the use of permutation-invariant set functions, which, in \glspl{np}, map from the set of observations to some representation space. A popular choice of architecture for these set functions are transformers \citep{vaswani2017attention,lee2019set}, giving rise to the family of \glspl{tnp} \citep{nguyen2022transformer, kim2019attentive, feng2022latent}.

%An important challenge in meta-learning is that of task heterogeneity, in which different datasets exhibit different characteristics. 
%In the context of meta-learning, practitioners might have access to related datasets that contain information about the characterstics of the dataset of interest. If these could be integrated into the NP, its predictions would improve. For example, a meta-dataset of PDE-simulated Navier-Stokes equations could be partitioned into smaller sets of datasets with the same Reynolds number. 
In the context of meta-learning, practitioners might have access to related datasets that share similarities with the dataset of interest. If these could be integrated into the NP, its predictions would improve. For example, a meta-dataset of PDE-simulated Navier-Stokes equations could be partitioned into smaller sets of datasets with the same Reynolds number.
Meta-learning a single \gls{np} can only recover the predictive distribution of the marginal---or average---stochastic process used to generate the entire meta-dataset. 
Instead, if we had access to sufficiently many datasets for each possible Reynolds number, a strictly preferable approach would be to meta-learn multiple \glspl{np}---one corresponding to each possible stochastic process (e.g..\ Reynolds number). While this would help us recover more accurate predictive distributions, it is often infeasible in practice due to either limited data availability for each possible stochastic process or the existence of a large number of possible stochastic processes.
%

% Instead, if we meta-learn multiple \glspl{np} for each possible stochastic process (i.e.\ Reynolds number), then we can recover more accurate predictive distributions if we know a-priori which stochastic process a dataset was drawn from. Unfortunately, if there are many possible stochastic processes, then this is infeasible, as we would require many datasets from each stochastic process to train a unique \gls{np}.

Rather than learning an individual \gls{np} for each stochastic process, we propose to learn a single \gls{np} that is able to condition on additional datasets that are known to be drawn from the same stochastic process. Intuitively, this can be thought of as amortising the learning of stochastic-process specific \glspl{np}, and we refer to this form of meta-learning as \emph{in-context in-context learning}. Constructing a model that is able to condition on sets of datasets, in addition to sets of observations, requires a function that is able to operate on sets of sets. Recently, there has been a string of research into transformer architectures which perform cross-attention between different data modalities \citep{jaegle2021perceiver,borgeaud2022improving,zhang2022crossformer,kim2021multi,shen2024episodic,xu2022mtformer,xu2023multimodal}. Our key insight is that we can use the same methods to construct a function that operates on a set of sets, providing a model architecture for in-context in-context learning. In doing so, we develop the \gls{model}. We highlight our key contributions as follows:

\begin{enumerate}
    \item %We establish the benefits of in-context in-context learning for modelling datasets drawn from a mixture of stochastic processes.
    We show that in-context in-context learning can lead to substantially improved predictive performance.
    \item We develop a new member of \glspl{tnp} for in-context in-context learning, the \gls{model}. At the core of the \gls{model} are pseudo-token based transformer architectures, which enable scaling up to large datasets.
    \item We provide an empirical investigation into the \gls{model}, demonstrating that it: i) does not reduce the performance of a regular \gls{pt-tnp} model when in-context datasets \emph{are not observed}; and ii) improves predictive performance over a regular \gls{pt-tnp} model when in-context datasets \emph{are observed}.
\end{enumerate}

% In such settings standard meta-learning models, such as \glspl{np}, are insufficient if we know a-priori which cluster a dataset belongs to. This is because they learn the prediction map of the \emph{marginal stochastic process}, rather than that of cluster-specific stochastic processes. Intuitively, we should be able to condition our predictions on other datasets that we know belong to the same cluster, as these datasets are informative about the cluster-specific stochastic process.

% Constructing a model that is able to condition on arbitrary numbers of different datasets is fundamentally a very similar challenge to being able to condition on arbitrary numbers of datapoints. We require a function that is able to operate on sets of sets. Recently, there has been a string of research into transformer architectures which perform cross-attention between different data modalities: examples of which include Perceiver \citep{jaegle2021perceiver}, RETRO \citep{betker2023improving} and Crossformer \citep{zhang2022crossformer}. In addition to self-attention amongst a set of tokens representing a single dataset, these methods utilise cross-attention to transfer information between different modalities. Our key insight is that this can be seen as a function on sets of sets, and so a similar architecture can be used to tackle the challenge we are considering.

% In this paper, we develop the \gls{model}.

\section{Background}
\label{sec:background}
Throughout this section, we use the following notation. Let $\mcX = \R^{D_x}$ and $\mcY = \R^{D_y}$ denote the input and output spaces, and let $(\bfx, \bfy) \in \mcX \times \mcY$ denote an input-output pair. Let $\mcS = \bigcup_{N=0}^{\infty} \left(\mcX \times \mcY\right)^N$ be a collection of all finite data sets, which includes the empty set $\varnothing$. We denote a context and target set with $\mcD_c, \mcD_t \in \mcS$, where $|\mcD_c| = N_c$ and $|\mcD_t| = N_t$. Let $\bfX_c \in \R^{N_c \times D_x}$, $\bfY_c \in \R^{N_c \times D_y}$ be the inputs and corresponding outputs of $\mcD_c$, with $\bfX_t \in \R^{N_t\times D_x}$, $\bfY_t \in \R^{N_t \times D_y}$ defined analogously. We denote a single task as $\tau = (\mcD_c, \mcD_t) = ((\bfX_c, \bfY_c), (\bfX_t, \bfY_t))$. Let $\mcP(\mcX)$ denote the collection of $\mcY$-valued stochastic processes on $\mcX$. Let $\Theta$ denote the parameter space of predictive distributions over the outputs. Let $\mcZ$ denote some latent space.

\subsection{Neural Processes}
\label{subsec:np_background}
\glspl{np} \citep{garnelo2018conditional, garnelo2018neural} are a type of meta-learning model which seek to learn the \emph{posterior prediction map} $\pi_P: \mcS \rightarrow \mcP(\mcX)$, which maps from context sets $\mcD_c$ to the posterior predictive distribution over the target outputs $p(\bfY_t | \bfX_t, \mcD_c)$ under the ground-truth stochastic process $P$. \gls{np} architectures generally consist of an \emph{encoder} $e: \mcS \times \mcX \rightarrow \mcZ$, which maps from $\mcD_c$ and $\bfX_t$ to some latent representation, and a \emph{decoder} $d: \mcX \times \mcZ \rightarrow \Theta$, which takes the representation and $\bfX_t$ as inputs and maps to the parameters of the predictive distribution over the target outputs: $p(\bfY_t | \bfX_t, \mcD_c) = p(\bfY_t | d(\bfX_t, e(\bfX_t, \mcD_c)))$. In this work, we limit our attention to \glspl{cnp} \citep{garnelo2018conditional}, which factorise the predictive distribution as $p(\bfY_t | \bfX_t, \mcD_c) = \prod_{n=1}^{N_t} p(\bfy_{t, n} | d(\bfx_{t, n}, e(\bfx_{t, n}, \mcD_c)))$. \glspl{cnp} are trained by maximising the posterior predictive likelihood:
\begin{equation}
    \label{eq:cnp-objective}
    \mcL_{\text{ML}} = \Exp{p(\tau)}{\sum_{n=1}^{N_t} \log p(\bfy_{t, n} | d(\bfx_{t, n}, e(\bfx_{t, n}, \mcD_c))}.
\end{equation}
Here, the expectation is taken over $p(\tau)$. In practice, we often only have access to a finite number of tasks, in which case we can replace the expectation with a Monte-Carlo estimate.

% \glspl{np} \citep{garnelo2018conditional, garnelo2018neural} aim to learn the mapping from context sets $\mcD_c$ to ground truth posterior distributions over the target outputs, $\mcD_c \mapsto p(\bfY_t | \bfX_t, \mcD_c)$, using meta-learning. This mapping is known as the \emph{posterior prediction map} $\pi_P: \mcS \rightarrow \mcP(\mcX)$, where $P$ denotes the ground truth stochastic process over functions mapping from $\mcX$ to $\mcY$.
% % take as input a set of input-output pairs, the context set $\mc\mcD_c = \{\bfx_{c, n} \in \mcX, \bfY_{c, n} \in \mcY \}_{n=1}^{N_c}$, and output a posterior predictive distribution over the target outputs $\bfY_t = \{\bfy_{t, n}\}_{n=1}^{N_t}$ at the target locations $\bfX_t = \{\bfx_{t, n}\}_{n=1}^{N_t}$: $p(\bfY_t | \bfX_t, \mc\mcD_c)$. 
% Common to all \gls{np} architectures is an encoder and decoder. The encoder maps from $\mcD_c$ and $\bfX_t$ to some representation, $e(\mcD_c, \bfX_t)$.\footnote{In many \gls{np} architectures, including the original \gls{cnp} and \gls{np}, the representation does not depend on the target inputs $\bfX_t$.} The decoder takes as input the representation and target inputs $\bfX_t$ and outputs $d(\bfX_t, e(\mcD_c, \bfX_t))$, which are the parameters of the predictive distribution over the target outputs $\bfY_t$: $p(\bfY_t | \bfX_t, \mcD_c) = p(\bfY_t | d(\bfX_t, e(\mcD_c, \bfX_t)))$.

\subsection{Transformer Neural Processes}
\label{subsec:tnp_background}
Transformers can be understood as permutation-equivariant set functions \citep{lee2019set}. This makes their use in \glspl{np} natural, since we require a set function to construct the mapping from context sets to predictive distributions. For the family of \glspl{tnp}, this is generally achieved by: 1) obtaining an initial token representation for the context points, $\bfZ^0_c \in \R^{N_c \times D_z}$, and target input locations, $\bfZ^0_t \in \R^{N_t \times D_z}$, using point-wise embeddings; 2) passing the union of the initial context and target tokens, $\bfZ^0 = [\bfZ^0_c, \bfZ^0_t]$, through a transformer-style architecture; and 3) passing the output tokens (corresponding to the target inputs) of the final layer $L$ of the \gls{tnp}, $\bfZ^L_t$, through another MLP to obtain the parameters of the predictive distribution $p(\bfY_t | \bfX_t, \mcD_c)$. We provide a thorough description of the operations used in transformer-style architectures in \Cref{app:mhsa_mhca}.

The form of transformer-style architecture varies between members of the family of \glspl{tnp}---we provide a description of several members of the family in \Cref{app:tnp_architectures}.
% This architecture consists of \gls{mhsa} layers operating on the context tokens, followed by \gls{mhca} layers operating between the context and target tokens. This has the effect of making the output target tokens $\bfZ^L_t$ conditionally independent of each other given the context tokens, resulting in a mean-field predictive distribution: $p(\bfY_t | \bfX_t, \mcD_c) = \prod_{n = 1}^{N_t} p(\bfy_{t, n} | \bfx_{t, n}, \mcD_c)$.
Whilst the regular \gls{tnp} has many advantages over other \gls{np} variants, it is hindered by its large computational complexity of $\order{N_c^2 + N_cN_t}$ given by the interaction of each token with the entire set of context tokens. This limits its application to relatively small datasets.  However, the computational complexity of the standard transformer can be reduced through the introduction of \emph{pseudo-tokens}, leading to pseudo-token based transformers. Let $\bfU \in \R^{M \times D_z}$ denote an initial set of $M \ll N_c$ tokens we refer to as pseudo-tokens. By only allowing tokens to interact with the set of context tokens $\bfZ_c$ indirectly through the smaller set of pseudo-tokens, we are able to reduce the computational complexity to $\order{MN_c + MN_t + M^2}$. We illustrate the architecture for the perceiver-style approach \citep{jaegle2021perceiver,feng2022latent} and IST-style approach \citep{lee2019set} of pseudo-token based transformers in \Cref{app:tnp_architectures}. We refer to the family of \glspl{tnp} that use a pseudo-token based transformer architecture as \glspl{pt-tnp}.

\section{In-Context In-Context Learning}
\label{sec:icicl}
In this section, we define the paradigm of in-context in-context learning. Our key result is provided in \Cref{thm:icicl}, which informally states that when provided with multiple datasets drawn from the same stochastic process, we can improve the quality of predictions by taking these datasets into account.
% we should condition on them when making predictions. 
We then propose a meta-learning model for achieving this in \Cref{sec:icicl_tnp}, the \gls{model}.

\subsection{In-Context In-Context Learning for Mixtures of Stochastic Processes}
\label{subsec:hpm_background}
%The goal
One of the goals of probabilistic meta-learning can be understood as modelling the posterior prediction map $\pi_P: \mcS \rightarrow \mcP(\mcX)$, where $P$ denotes the ground truth stochastic process over functions mapping from $\mcX$ to $\mcY$ that our datasets are sampled from. Yet, in many cases $P$ is itself a mixture of stochastic processes, and our datasets can be partitioned into samples from each stochastic process mixture. Concretely, we can model each dataset as 
\begin{equation}
        \xi_i \sim p(\xi), \quad \mcD_i \sim P(\xi_i)
\end{equation}
where $P(\xi_i) \in \mcP(\mcX)$ is a stochastic process defined by the latent variable $\xi_i \in \Xi$ from which dataset $\mcD_i$ is sampled. Consider the setting in which we have a set of additional datasets $\{\mcD_j\}$ also drawn from $P(\xi_i)$. Intuitively, the additional datasets provide information about $P(\xi_i)$, the stochastic process from which $\mcD_i$ is sampled. Conditioning on these additional datasets would therefore reduce our uncertainty of the ground-truth stochastic process from which $\mcD_i$ is sampled, and can therefore improve our approximation of the prediction map $\pi_{P(\xi_i)}$. This is formalised in the following theorem:
\begin{theorem}[In-context in-context learning]
    \label{thm:icicl}
    Let $\xi_i \sim p(\xi)$, $\mcD_i, \{\mcD_j\} \sim P(\xi_i)$. Let $p(\bfy | \bfx, \mcD_i, \xi_i)$ be the marginal posterior distribution of $P(\xi_i)$ given $\mcD_i$, $p(\bfy | \bfx, \mcD_i, \{\mcD_j\})$ be the marginal posterior distribution of the stochastic process $P$ given $\mcD_i$ and $\{\mcD_j\}$, and $p(\bfy | \bfx, \mcD_i)$ be the marginal posterior distribution of the stochastic process $P$ given $\mcD_i$. Then,
    \begin{equation}
        \Exp{\mcD_i, \{\mcD_j\}, \xi_i}{\KL{p(\bfy | \bfx, \mcD_i, \xi_i)}{p(\bfy | \bfx, \mcD_i, \{\mcD_j\})}} \leq \Exp{\mcD_i, \xi_i}{\KL{p(\bfy | \bfx, \mcD_i, \xi_i)}{p(\bfy | \bfx, \mcD_i)}}.
    \end{equation}
\end{theorem}

\begin{sproof}
    Observe that $\Exp{\mcD_i, \{\mcD_j\}, \xi_i}{\KL{p(\bfy | \bfx, \mcD_i, \xi_i)}{p(\bfy | \bfx, \mcD_i, \{\mcD_j\})}}$ can be expressed as
    \begin{equation}
        \begin{aligned}
            &\Exp{\mcD_i, \{\mcD_j\}, \xi_i}{\KL{p(\bfy | \bfx, \mcD_i, \xi_i)}{p(\bfy | \bfx, \mcD_i, \{\mcD_j\})}} \\
            &\qquad = -\Exp{\mcD_i, \xi_i}{\entropy{p(\bfy | \bfx, \mcD_i, \xi_i)}} - \Exp{\mcD_i, \{\mcD_j\}, \bfy}{\log p(\bfy | \bfx, \mcD_i, \{\mcD_j\}} \\
            &\qquad = -\Exp{\mcD_i, \xi_i}{\entropy{p(\bfy | \bfx, \mcD_i, \xi_i)}} + \Exp{\mcD_i}{\Exp{\{\mcD_j\} | \mcD_i}{\entropy{p(\bfy | \bfx, \mcD_i, \{\mcD_j\}}}} \\
            &\qquad \leq -\Exp{\mcD_i, \xi_i}{\entropy{p(\bfy | \bfx, \mcD_i, \xi_i)}} + \Exp{\mcD_i}{\entropy{p(\bfy | \bfx, \mcD_i)}} \\
            &\qquad = \Exp{\mcD_i, \xi_i}{\KL{p(\bfy | \bfx, \mcD_i, \xi_i)}{p(\bfy | \bfx, \mcD_i})}
        \end{aligned}
    \end{equation}
    The inequality holds as $\entropy{p(\bfy | \bfx, \mcD_i, \{\mcD_j\}} \leq \entropy{p(\bfy | \bfx, \mcD_i)} \ \forall \{\mcD_j\}$, where $\entropy{\cdot}$ denotes the entropy. See \Cref{app:icicl_proof} for a detailed proof.
\end{sproof}
Note that $p(\bfy_t | \bfx_t, \mcD_i, \{\mcD_j\}, \xi_i) = p(\bfy_t | \bfx_t, \mcD_i, \xi_i)$, as $\mcD_i$ and $\{\mcD_j\}$ are conditionally independent given $\xi_i$. Since we do not observe $\xi_i$, \Cref{thm:icicl} tells us that we should target the predictive distribution $p(\bfy | \bfx, \mcD_i, \{\mcD_j\})$, instead of $p(\bfy | \bfx, \mcD_i)$. We refer to this form of learning as \emph{in-context in-context learning} (ICICL), and it requires models that are able to condition on a set of datasets, in addition to individual datasets. We refer to the additional set of datasets as the \emph{in-context datasets}, and shall denote this set as $\{\mcD_{ic, j}\}_{j=1}^{N_{ic}}$, where $N_{ic}$ is the number of in-context datasets, with $|\mcD_{ic, j}| = N_{ic, j}$.

\subsection{In-Context In-Context Learning with Transformer Neural Processes}
\label{sec:icicl_tnp}
\begin{figure}[htb]
    \centering
    \includegraphics[width=\textwidth]{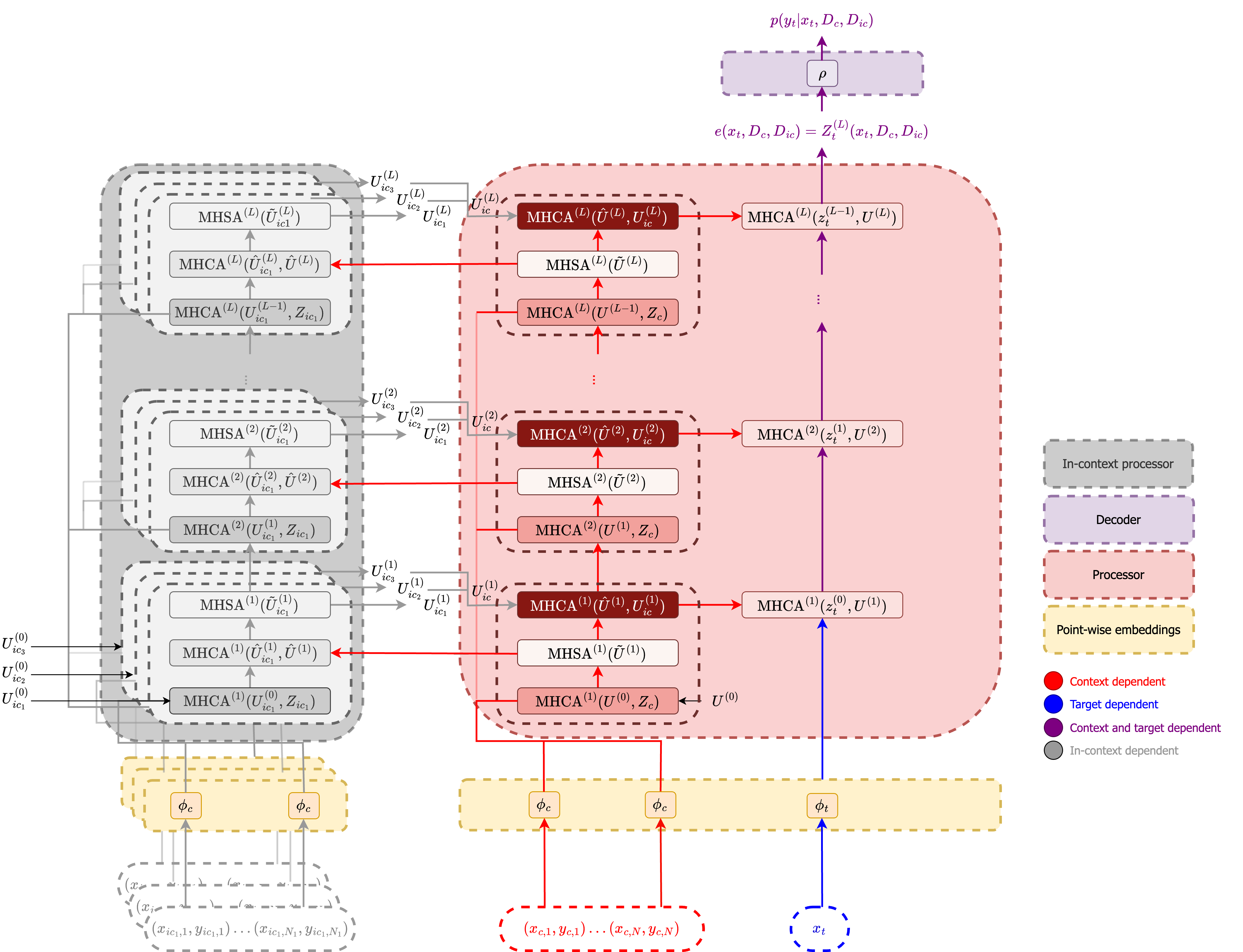}
    \caption{A diagram illustrating the \gls{model} architecture with three in-context datasets. The point-wise embedding layer is used to get an initial token representation of all datapoints, including the target input location $\bfx_t$. Then for each layer of the processor, pseudo-token representations for each of the in-context datasets, $\bfU_{ic}$, and the context dataset, $\bfU$, are updated through \gls{mhca} operations. The in-context pseudo-tokens $\bfU_{ic}$ are then modulated by the context pseudo-tokens $\bfU$, followed by \gls{mhsa} operations on each set of pseudo-tokens. The in-context pseudo-tokens then modulate the context pseudo-tokens, and finally the context pseudo-tokens modulate the token representation of the target input, $\bfz_t$. After $L$ layers, the processor outputs the encoder representation $e(\bfx_t, \mcD_c, \{\mcD_{ic, j}\}_{j=1}^{j=N_{ic}})$.}
    \label{fig:icicl-tnp}
\end{figure}

Whilst current \gls{tnp} architectures are capable of conditioning on individual datasets, they are not capable of conditioning on an arbitrary number of additional in-context datasets in order to approximate the predictive distribution $p(\bfy | \bfx, \mcD_i, \{\mcD_{ic, j}\}_{j=1}^{N_{ic}})$. Fortunately, the flexibility afforded by \gls{mhca} layers makes this extension natural. As the total number of datapoints within each dataset is potentially large, we shall make use of the pseudo-token based architectures discussed in \Cref{subsec:tnp_background}. There exists a number of possible choices of architecture, but unless stated otherwise, we shall use the architecture illustrated in \Cref{fig:icicl-tnp}, with a description of other choices of architecture provided in \Cref{app:icicl_tnp}. We refer to this family of \glspl{pt-tnp} as \glspl{model}. 

At a high level, \glspl{model} perform the following set of operations: 1) construct initial token representations of each datapoint within each dataset using point-wise encodings; 2) construct a pseudo-token representation for each of the in-context datasets, $\bfU_{ic, j}$, and context dataset, $\bfU$; 3) perform cross-attention between the in-context pseudo-tokens and context pseudo-tokens. We note that this architecture forms a valid set function \citep{zaheer2017deep,wagstaff2022universal} on the context datapoints, a valid set function on the set of in-context datasets, and a valid set function on the datapoints within each in-context dataset (see \Cref{app:valid_set_function} for proof). The \gls{model} has an asymptotic computational complexity $\order{MN_c + MN_t + \sum_{j=1}^{N_{ic}}\left(M_{ic}N_{ic, j} + MM_{ic} \right)}$, where $N_{ic}$ denotes the number of in-context datasets, $N_{ic, j}$ denotes the number of datapoints in the $j$-th in-context dataset, and $M_{ic}$ denotes the number of pseudo-tokens used for each in-context dataset. Importantly, this scales linearly in $N_c$, $N_t$, $N_{ic}$ and $N_{ic, j}$, enabling it to scale to both large datasets and many in-context datasets.

The \gls{model} is trained in an analogous fashion to \glspl{cnp}, whereby the model parameters are optimised by maximising the predictive likelihood:
\begin{equation}
    \mcL_{\text{ML}} = \Exp{p(\tau)}{\sum_{n=1}^{N_t}\log p(\bfy_{t, n} | d(\bfx_{t, n}, e(\bfx_{t, n}, \mcD_c, \{\mcD_{ic}\})}.
\end{equation}
Here, tasks $\tau = (\mcD_c, \{\mcD_{ic}\}, \mcD_t)$ contain additional in-context datasets, which are fed into the encoder. 

\section{Related Work}
\label{sec:related-work}
\paragraph{Cross-Attention Based Architectures in NLP}
Although not described as in-context in-context learning, similar architectures have been proposed for conditioning on additional datasets in NLP. Notably, the retrieval-enhanced transformer (RETRO) \citep{borgeaud2022improving} performs cross-attention between the token representations of additional `retrieved' texts, demonstrating an  improvement in performance over the original transformer architecture \citep{vaswani2017attention}. However, the architecture used in RETRO differs substantially from ours. Rather than conditioning on entire datasets, RETRO performs chunking of text to perform next chunk prediction, in which other chunks similar to the previous chunks are used to modulate the predictions. Further, they do not make use of pseudo-token-based transformer architectures to enable scaling to chunks containing more than 64 tokens.

\paragraph{Cross-Attention Based Architectures for Multi-Task Learning}
A closely related family of transformer architectures are those used for multi-task learning \citep{zhang2022crossformer,kim2021multi,shen2024episodic,xu2022mtformer} and multi-modal learning \citep{jaegle2021perceiver,xu2023multimodal}. An important difference between this style of architecture and that used for in-context in-context learning is that exchangeability of the in-context datasets is not required, nor desirable. This owes to the fact that multi-output and multi-modal datasets are not independent samples from the same underlying stochastic process. Rather, they can be understood as multiple outputs of the same sample from some stochastic process. The difference is subtle, and in practice is achieved through positional encodings for each output of multi-output data, and different tokenisations of each mode in multi-modal data.

\paragraph{Conditioning on Exchangeable Datasets in Causal ML} Finally, the causal structure induction with a supervised approach (CSIvA) model from \cite{ke2022learning} also shares similarities with our approach. Similar to the \gls{model}, the CSIvA conditions on multiple datasets in an exchangeable manner to make inference. The two models differ significantly elsewhere, however. Whereas we are interested in predictive distributions over outputs, \cite{ke2022learning} are interested in inferring causal structure from observational and interventional data. These differences materialise in their use of positional encoding to indicate the identity of the causal node of observations. In a sense, this model bears more resemblance to the multi-task learning architectures, with the additional ability to condition on many instances of samples from the stochastic process in a single forward pass.

\section{Experiments}
\label{sec:experiments}
In this section, we investigate the performance of the \gls{model}. We seek to answer two questions: 1) can the \gls{model} recover the performance of a regular \gls{pt-tnp} when no in-context datasets are provided; 2) how does the performance of the \gls{model} vary with the number of in-context datasets provided. In each experiment, we compare the performance of the \gls{model} with a regular \gls{pt-tnp} equivalent. We also provide results for a version of the \gls{cnp} that supports ICICL (see \Cref{app:icicl_cnp} for details of the architecture) alongside the regular \gls{cnp}. We provide more thorough experimental details in \Cref{app:experiment_details}, including the choice of model architecture and model training.

\subsection{Synthetic Regression}
We consider a synthetic 1-D regression task using samples drawn from Gaussian processes (GPs) with different kernel types and different kernel hyperparameters. First, we sample either a radial basis function (RBF) or periodic kernel. The kernel hyperparameter $\ell$---corresponding to the lengthscale and period for the RBF and periodic kernel, respectively---is sampled as $\log \ell \sim \mcU_{[\log 0.25, \log 4]}$. This is shared between the context and in-context datasets. For each context dataset, the number of context points is drawn according to $N_c \sim \mcU\{1, 64\}$, while the number of target points is set to $N_t=128$. The context and target inputs are sampled from $x_c \sim \mcU_{[-2, 2]}$ and $x_t \sim \mcU_{[-4, 4]}$. For each such context dataset, we sample $N_{ic}$ in-context datasets with $N_{ic} \sim \mcU\{0, 5\}$, where the number of datapoints for each in-context dataset is sampled as $N_{ic, j} \sim \mcU\{64, 128\}$. The in-context inputs are sampled according to $x_{ic} \sim \mcU_{[-4, 4]}$. The observation noise is set to 0.2, and the test set consist of $80,000$ datasets.

\Cref{fig:ic-gp-performance} shows that when no in-context information is available, the \gls{model} achieves similar performance to that of the \gls{pt-tnp} (within one standard deviation), and significantly outperforms the \gls{cnp}. Conditioning on a single in-context dataset significantly improves the predictive performance of the \gls{model}, with the improvements in performance plateauing as the number of in-context datasets is increased further. Importantly, these results demonstrate that 1) the \gls{model} is able to recover the performance of the regular \gls{pt-tnp} when no in-context datasets are provided; and 2) the \gls{model} is able to perform in-context in-context learning effectively.

We compare the predictive distributions of the \gls{model} and \gls{pt-tnp} in \Cref{fig:icicl-gp-figs}. With only 10 context points, the \gls{pt-tnp} has insufficient information to infer the periodicity of the underlying stochastic process, so it fails to both extrapolate beyond or interpolate between the data effectively. In contrast, when an in-context dataset drawn from the same periodic kernel is provided, the \gls{model} is able to accurately model the posterior predictive map of the ground-truth periodic stochastic process. Not only does this demonstrate the importance of in-context in-context learning, it also demonstrates the effectiveness of the \gls{model} in realising it. 
%Moreover, the periodicity extends outside of the training range ($[-2, 2]$) as well. 
Additional experiments are provided in \Cref{app:synthetic-1d}, illustrating the effect on the predictions of the \gls{model} of conditioning on in-context datasets drawn from a different stochastic process.

We also trained an ICICL-\gls{cnp} model on this task, but we did not observe any benefits from in-context in-context learning. This is in contrast to what we observe using \gls{model} on this task, as well as what we observe in \Cref{sec:image_completion}. Thus, we hypothesise that the failure of the ICICL-\gls{cnp} model to perform in-context in-context learning on GP synthetic regression arises from its limited capacity relative to \gls{model} and from the difficulty of the task.

\paragraph{Out-of-distribution (OOD) testing} We also investigated the behaviour of the \gls{model}, as compared to \gls{pt-tnp}, when tested OOD--during testing we sampled the kernel hyperparameter according to $\log \ell \sim \mcU_{[\log 0.1, \log 0.25]  \cup [\log 4, \log 10]}$ for both the context and in-context datasets. \Cref{fig:ic-gp-performance-OOD} shows that in-context learning improves predictive performance even if at test time the samples come from stochastic processes the model has not been trained on. This is reflected in \Cref{fig:icicl-gp-figs-OOD}, where the context datapoints come from a GP with a periodic kernel and $\ell = 6.08$. The \gls{model} is better than the \gls{pt-tnp} at capturing in its uncertainty the slowly-varying variations characteristic to this kernel. We show more examples in \Cref{app:synthetic-1d}.

\begin{figure}[htbp]
\begin{minipage}[t]{0.45\textwidth}
    \begin{minipage}[t]{\linewidth}
        %\centering
        \begin{tabular}{lc} 
        \toprule
        \textbf{Model}     & \multicolumn{1}{l}{\textbf{Log lik. ($\uparrow$)}}  \\
        \midrule
        CNP  & $-0.812 \pm 0.005$ \\
        \gls{pt-tnp} & $-0.598 \pm 0.005$ \\
        \gls{model} (0)  & $-0.607 \pm 0.005$ \\
        \gls{model} (1) & $-0.499 \pm 0.005$ \\
        \gls{model} (2) & $-0.474 \pm 0.005$ \\
        \gls{model} (3) & $-0.469 \pm 0.005$ \\
        \gls{model} (4) & $-0.467 \pm 0.005$ \\
        \gls{model} (5)& $\mathbf{-0.466 \pm 0.005}$ \\
        \bottomrule
        % \caption{Comparison of the predictive performance (in terms of test log likelihood) between the \gls{cnp}, \gls{pt-tnp}, and the \gls{model} with varying number of in-context datasets (indicated within brackets).}
        % \label{fig:ic-gp-performance}
        \end{tabular}
        \captionof{table}{Comparison of the predictive performance (in terms of test log likelihood) between the \gls{cnp}, \gls{pt-tnp}, and the \gls{model} with varying number of in-context datasets (indicated within brackets).}
        \label{fig:ic-gp-performance}
    \end{minipage}
\end{minipage}%
\hfill
\begin{minipage}[t]{0.53\textwidth}
\vspace{-2.6cm}
      {%
        \subfigure[Regular \gls{pt-tnp}.][c]{\label{fig:lbanp-ic-gp}%
          \includegraphics[width=\linewidth,trim={0.5cm 0.5cm 0.5cm 0.5cm},clip]{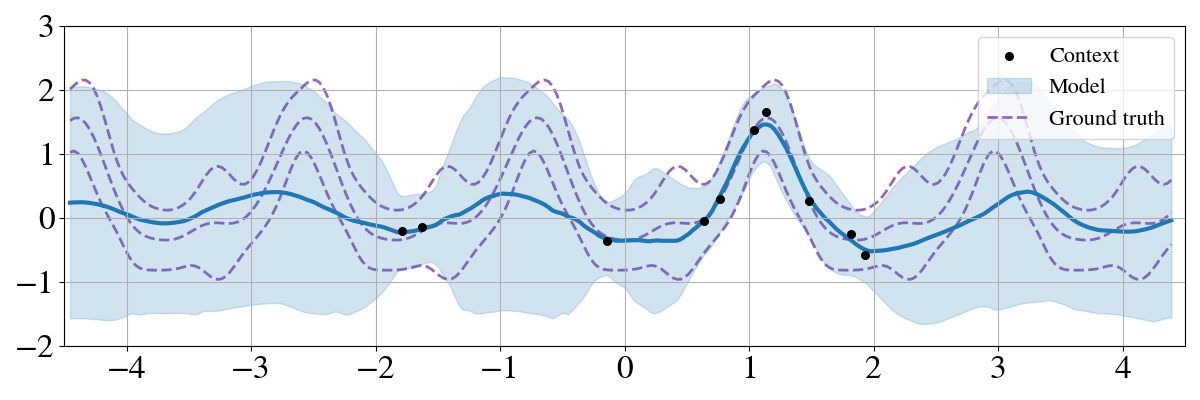}}%
        \\
        \subfigure[\gls{model}.][c]{\label{fig:iclbanp-ic-gp-0}%
          \includegraphics[width=\linewidth,trim={0.5cm 0.5cm 0.5cm 0.5cm},clip]{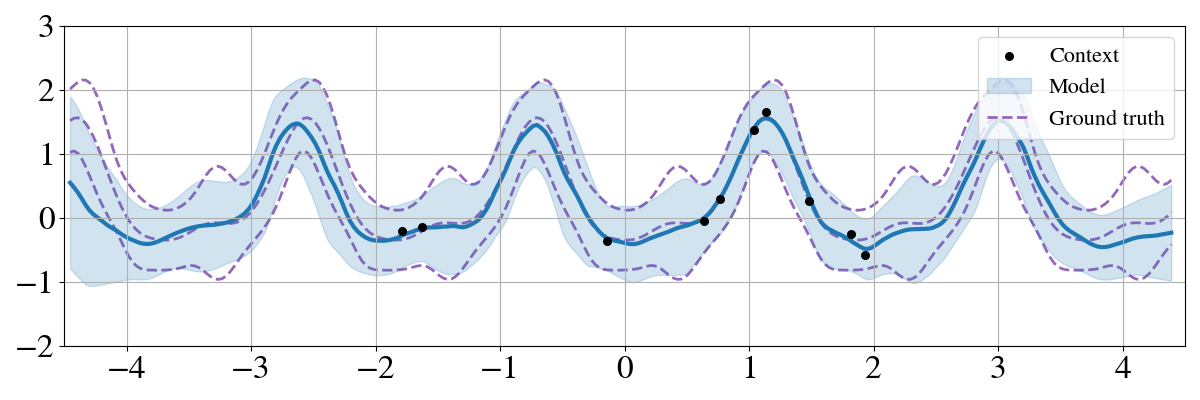}}%
      }
      {\caption{The difference between the predictive distributions of the regular \gls{pt-tnp} and the \gls{model} when conditioning on three in-context datasets with 128 datapoints (not shown here).}
      \label{fig:icicl-gp-figs}}
     
\end{minipage}
\end{figure}

\begin{figure}[htbp]
\begin{minipage}[t]{0.45\textwidth}
    \begin{minipage}[t]{\linewidth}
        %\centering
        \begin{tabular}{lc} 
        \toprule
        \textbf{Model}     & \multicolumn{1}{l}{\textbf{Log lik. ($\uparrow$)}}  \\
        \midrule
        CNP  & $-0.880 \pm 0.006$ \\
        \gls{pt-tnp} & $-0.798 \pm 0.007$ \\
        \gls{model} (0)  & $-0.783 \pm 0.007$ \\
        \gls{model} (1) & $-0.721 \pm 0.006$ \\
        \gls{model} (2) & $-0.702 \pm 0.006$ \\
        \gls{model} (3) & $\mathbf{-0.700 \pm 0.006}$ \\
        \bottomrule
        % \caption{Comparison of the predictive performance (in terms of test log likelihood) between the \gls{cnp}, \gls{pt-tnp}, and the \gls{model} with varying number of in-context datasets (indicated within brackets).}
        % \label{fig:ic-gp-performance}
        \end{tabular}
        \captionof{table}{Comparison of the predictive performance (in terms of test log likelihood) when tested OOD between the \gls{cnp}, \gls{pt-tnp}, and the \gls{model} with varying number of in-context datasets (indicated within brackets).}
        \label{fig:ic-gp-performance-OOD}
    \end{minipage}
\end{minipage}%
\hfill
\begin{minipage}[t]{0.53\textwidth}
\vspace{-2.1cm}
      {%
        \subfigure[Regular \gls{pt-tnp}.][c]{\label{fig:lbanp-ic-gp-OOD}%
          \includegraphics[width=\linewidth,trim={0.5cm 0.5cm 0.5cm 0.5cm},clip]{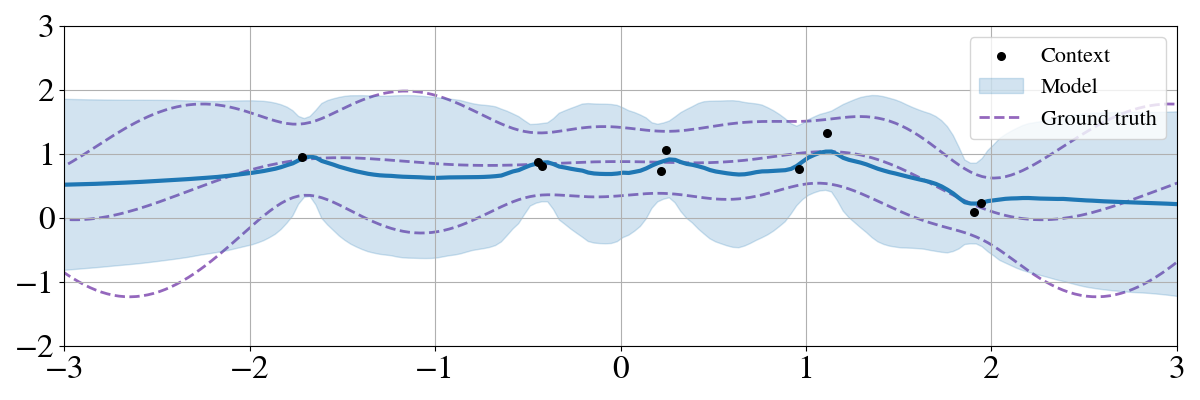}}%
        \\
        \subfigure[\gls{model}.][c]{\label{fig:iclbanp-ic-gp-OOD}%
          \includegraphics[width=\linewidth,trim={0.5cm 0.5cm 0.5cm 0.5cm},clip]{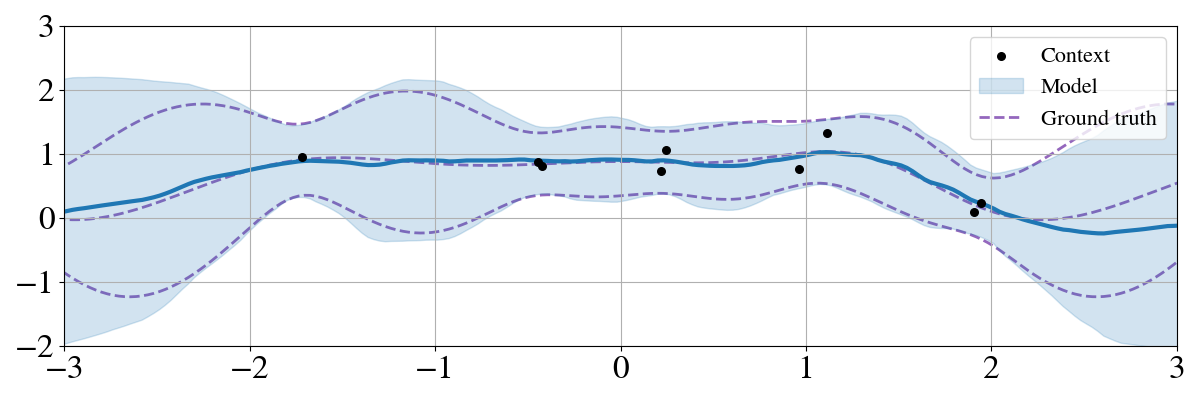}}%
      }
      {\caption{The difference between the predictive distributions when tested OOD of the regular \gls{pt-tnp} and the \gls{model} when conditioning on three in-context datasets. The context datapoints come from a GP with a periodic kernel with $\ell = 6.08$.}
      \label{fig:icicl-gp-figs-OOD}}
\end{minipage}
\end{figure}

\subsection{Image Completion}
\label{sec:image_completion}
We consider an image in-painting experiment using the MNIST dataset \citep{lecun1998gradient}. Each MNIST image can be interpreted as spatial regression of pixel values $\bfy_n \in \R$ given a 2-D pixel location $\bfx_n \in \R^2$. We construct context datasets by sampling the number of pixels according to $N_c \sim \mcU\{N/100, N/5\}$, where $N=784$ denotes the total number of pixels in each image, and set $N_t = N - N_c$. For each context dataset, we sample in-context datasets from other images of the same label with the number of in-context datasets sampled as $N_{ic} \sim \mcU\{0, 3\}$, and sample the number of datapoints for each in-context dataset as $N_{ic, c} \sim \mcU\{N/100, N/2\}$. \Cref{fig:ic-mnist-performance} compares the predictive performance of the regular \gls{pt-tnp} with the \gls{model} as $N_{ic}$ and $N_c / N$ vary. We also provide results for the ICICL-\gls{cnp} and \gls{cnp}.

\begin{figure}
\begin{minipage}[b]{0.38\textwidth}
        \centering
        \includegraphics[width=\textwidth]{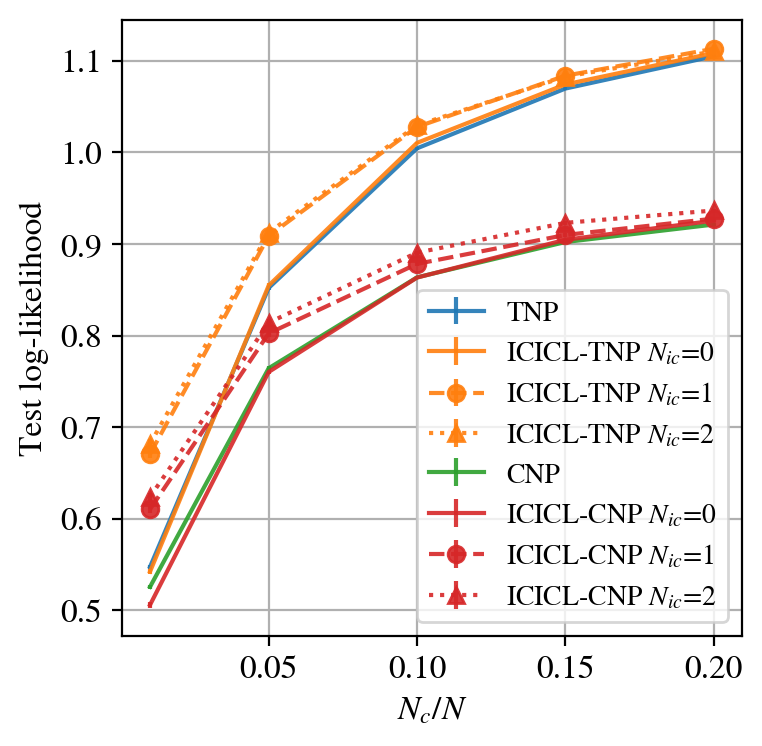}
        \caption{A comparison between the predictive performance of the \gls{model}, regular \gls{pt-tnp}, ICICL-\gls{cnp} and \gls{cnp} as the proportion of context datapoints $N_c / N$ varies in the MNIST in-painting experiment.}
        \label{fig:ic-mnist-performance}
\end{minipage}
\hfill
\begin{minipage}[b]{0.58\textwidth}
    \floatconts
      {fig:ic-mnist-comparison}
      % {\caption{A comparison between the predictive distribution of the \gls{model} and the regular \gls{tnp} when conditioning on the context distribution and in-context dataset shown. Observe that when conditioning on an in-context dataset from the same stochastic process as the context dataset (i.e.\ MNIST label 6), the \gls{model} is more confident in its predictive distribution. When the \gls{model} conditions on an in-context dataset from a different stochastic process (i.e.\ MNIST label 1), the predictive distribution is less confident.}}
      {\caption{A comparison between the predictive distribution of the \gls{model} and the regular \gls{pt-tnp} when conditioning on the context and in-context dataset shown.}}
      {%
        \subfigure[Regular \gls{pt-tnp}.][c]{\label{fig:lbanp-ic-mnist}%
          \includegraphics[width=\linewidth,trim={3cm 0.5cm 3cm 0.5cm},clip]{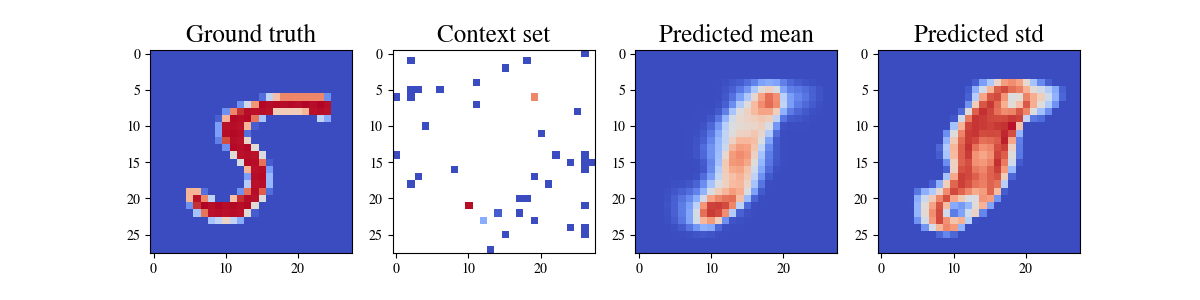}}%
        \\
        \subfigure[\gls{model}.][c]{\label{fig:iclbanp-ic-mnist-0}%
          \includegraphics[width=\linewidth,trim={3cm 0.5cm 3cm 0.5cm},clip]{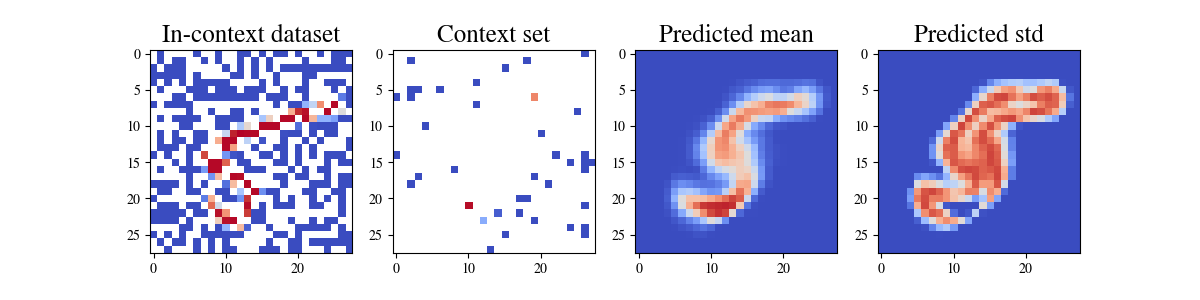}}%
      }
      % {%
      %   \subfigure[Regular \gls{tnp}.][c]{\label{fig:lbanp-ic-mnist}%
      %     \includegraphics[width=0.725\linewidth,trim={9.25cm 0 3cm 0},clip]{figs/lbanp-mnist-pred-000.png}}%
      %   %
      %   \\
      %   %
      %   \subfigure[\gls{model}][c]{\label{fig:iclbanp-ic-mnist-0}%
      %     \includegraphics[width=\linewidth,trim={3cm 0 3cm 0},clip]{figs/iclbanp-mnist-pred-000.png}}%
      %   %
      %   \\
      %   % %
      %   % \subfigure[\gls{model}: in-context dataset from \emph{different} stochastic process.][c]{\label{fig:iclbanp-ic-mnist-1}%
      %   % \includegraphics[width=\linewidth,trim={3cm 0 3cm 0},clip]{figs/iclbanp-mnist-pred-002.png}}
      % }
\end{minipage}
\end{figure}
% \begin{figure}[htbp]
%     \centering
%     \includegraphics[width=0.4\textwidth]{figs/mnist-ic-performance-plot-small.png}
%     \caption{A comparison between the predictive performance of the \gls{model}, regular \gls{tnp}, ICICL-\gls{cnp} and \gls{cnp} as the proportion of context datapoints $N_c / N$ varies. We include results for the ICICL-variants with $N_{ic} = [0, 1, 2]$.}
%     \label{fig:enter-label}
% \end{figure}

We observe that the \gls{model} recovers the performance of the \gls{pt-tnp} when conditioning on no in-context datasets, and that, as expected, the gains in performance of the \gls{model} plateau with increasing $N_{ic}$. In \Cref{fig:ic-mnist-comparison}, we compare the predictive distributions of the regular \gls{tnp} and \gls{model} for a single dataset. As very few pixels are observed in the context dataset, the regular \gls{tnp} is unconfident in its predictions, as the context dataset could correspond to datasets sampled from several MNIST digits. With the addition of an in-context dataset, the \gls{model} is able to infer pixel values much closer to the ground truth. We provide additional comparisons in \Cref{app:image_completion}, showcasing examples where the in-context dataset is drawn from a different stochastic process to the context dataset. 

\subsection{Environmental Data}
In this experiment, we model a real-world environmental regression problem derived from ERA5 \citep{cccs2020}. We consider the surface air temperature across both space and time within central Europe (latitude / longitude range of $[42^{\circ},\ 53^{\circ}]$ / $[8^{\circ},\ 28^{\circ}]$). During training, datasets spanning 18 hours (one sample every six hours) and $5^{\circ}$ (with a resolution of $0.5^{\circ}$) across each axis are sampled. In-context datasets are obtained from the same region, but non-overlapping regions in time. Once the spatio-temporal location has been sampled, the number of context points is sampled as $N_c \sim \mcU\{N/100, N/3\}$, where $N = 300$ denotes the maximum number of measurements in a single dataset. The number of in-context datasets is sampled as $N_{ic} \sim \mcU\{0, 2\}$, with the number of datapoints within each in-context dataset sampled as $N_{ic, j} \sim \mcU\{N/100, N/5\}$. The number of target points is set to $N_t = N - N_c$. We train on data from the first six months of 2019, and test on the latter six months. \Cref{tab:environmental_data} compares the predictive performance of the \gls{model} with the regular \gls{pt-tnp}. As with the previous two experiments, the \gls{model} is able to recover the performance of the \gls{pt-tnp} with no in-context datasets, and outperforms the \gls{pt-tnp} when in-context datasets are provided.

\begin{table}[htbp]
    \centering
    \begin{tabular}{l c}
    \toprule
         \textbf{Model} & \textbf{Test Log-Likelihood}  \\
         \midrule
         \gls{pt-tnp} & $1.15 \pm 0.01$ \\
         \gls{model} (0) & $1.15 \pm 0.01$ \\
         \gls{model} (1) & $1.18 \pm 0.01$ \\
         \gls{model} (2) & $1.19 \pm 0.01$ \\
         \bottomrule
    \end{tabular}
    \caption{Comparison of the test log-likelihood on the environmental data for the \gls{pt-tnp} and \gls{model} with varying number of in-context datasets (indicated within brackets).}
    \label{tab:environmental_data}
\end{table}

\section{Conclusion}
\label{sec:conclusion}
We have introduced the paradigm of in-context in-context learning, in which multiple datasets are used for predictive inference on a dataset drawn from the same stochastic process. We developed a \gls{np} for performing in-context in-context learning, the \gls{model}, which utilises pseudo-token-based transformer architectures for the computationally efficient handling of both context and in-context datasets with potentially many datapoints in each. Through a set of synthetic and real-world experiments, we demonstrate the importance of in-context in-context learning and the effectiveness of the \gls{model} in performing it. Further, we demonstrate that the \gls{model} is able to recover the performance of an equivalent \gls{pt-tnp} when no in-context datasets are provided. A requirement of in-context in-context learning generally is for practitioners to be able to identify datasets drawn from the same stochastic process, which in turn limits the applicability of the \gls{model}. Nonetheless, we believe that by establishing the benefit of in-context in-context learning, new machine learning applications can be developed that utilise it. For example, one can envision developing image generation tools which can be provided with additional images that the user wishes the generated sample to be similar to. We are excited by the abundance of potential directions in which to take this research, and look forward to exploring them in future work.
%with additional images that the user wishes the generated image to be similar to. We are excited by the abundance of potential directions in which to take this research, and look forward to exploring them in future work.

\acks{
CD is supported by the Cambridge Trust Scholarship. AW acknowledges support from a Turing AI fellowship under grant EP/V025279/1 and the Leverhulme Trust via CFI. RET is supported by gifts from Google, Amazon, ARM, Improbable and EPSRC grant EP/T005386/1.
}

\bibliography{bibliography}

\begin{thebibliography}{22}
\providecommand{\natexlab}[1]{#1}
\providecommand{\url}[1]{\texttt{#1}}
\expandafter\ifx\csname urlstyle\endcsname\relax
  \providecommand{\doi}[1]{doi: #1}\else
  \providecommand{\doi}{doi: \begingroup \urlstyle{rm}\Url}\fi

\bibitem[Borgeaud et~al.(2022)Borgeaud, Mensch, Hoffmann, Cai, Rutherford, Millican, Van Den~Driessche, Lespiau, Damoc, Clark, et~al.]{borgeaud2022improving}
Sebastian Borgeaud, Arthur Mensch, Jordan Hoffmann, Trevor Cai, Eliza Rutherford, Katie Millican, George~Bm Van Den~Driessche, Jean-Baptiste Lespiau, Bogdan Damoc, Aidan Clark, et~al.
\newblock Improving language models by retrieving from trillions of tokens.
\newblock In \emph{International conference on machine learning}, pages 2206--2240. PMLR, 2022.

\bibitem[{Copernicus Climate Change Service}(2020)]{cccs2020}
{Copernicus Climate Change Service}.
\newblock Near surface meteorological variables from 1979 to 2018 derived from bias-corrected reanalysis, 2020.

\bibitem[Feng et~al.(2022)Feng, Hajimirsadeghi, Bengio, and Ahmed]{feng2022latent}
Leo Feng, Hossein Hajimirsadeghi, Yoshua Bengio, and Mohamed~Osama Ahmed.
\newblock Latent bottlenecked attentive neural processes.
\newblock \emph{arXiv preprint arXiv:2211.08458}, 2022.

\bibitem[Foong et~al.(2020)Foong, Bruinsma, Gordon, Dubois, Requeima, and Turner]{foong2020meta}
Andrew Foong, Wessel Bruinsma, Jonathan Gordon, Yann Dubois, James Requeima, and Richard Turner.
\newblock Meta-learning stationary stochastic process prediction with convolutional neural processes.
\newblock \emph{Advances in Neural Information Processing Systems}, 33:\penalty0 8284--8295, 2020.

\bibitem[Garnelo et~al.(2018{\natexlab{a}})Garnelo, Rosenbaum, Maddison, Ramalho, Saxton, Shanahan, Teh, Rezende, and Eslami]{garnelo2018conditional}
Marta Garnelo, Dan Rosenbaum, Christopher Maddison, Tiago Ramalho, David Saxton, Murray Shanahan, Yee~Whye Teh, Danilo Rezende, and SM~Ali Eslami.
\newblock Conditional neural processes.
\newblock In \emph{International conference on machine learning}, pages 1704--1713. PMLR, 2018{\natexlab{a}}.

\bibitem[Garnelo et~al.(2018{\natexlab{b}})Garnelo, Schwarz, Rosenbaum, Viola, Rezende, Eslami, and Teh]{garnelo2018neural}
Marta Garnelo, Jonathan Schwarz, Dan Rosenbaum, Fabio Viola, Danilo~J Rezende, SM~Eslami, and Yee~Whye Teh.
\newblock Neural processes.
\newblock \emph{arXiv preprint arXiv:1807.01622}, 2018{\natexlab{b}}.

\bibitem[Jaegle et~al.(2021)Jaegle, Gimeno, Brock, Vinyals, Zisserman, and Carreira]{jaegle2021perceiver}
Andrew Jaegle, Felix Gimeno, Andy Brock, Oriol Vinyals, Andrew Zisserman, and Joao Carreira.
\newblock Perceiver: General perception with iterative attention.
\newblock In \emph{International conference on machine learning}, pages 4651--4664. PMLR, 2021.

\bibitem[Jha et~al.(2022)Jha, Gong, Wang, Turner, and Yao]{jha2022neural}
Saurav Jha, Dong Gong, Xuesong Wang, Richard~E Turner, and Lina Yao.
\newblock The neural process family: Survey, applications and perspectives.
\newblock \emph{arXiv preprint arXiv:2209.00517}, 2022.

\bibitem[Ke et~al.(2022)Ke, Chiappa, Wang, Goyal, Bornschein, Rey, Weber, Botvinic, Mozer, and Rezende]{ke2022learning}
Nan~Rosemary Ke, Silvia Chiappa, Jane Wang, Anirudh Goyal, Jorg Bornschein, Melanie Rey, Theophane Weber, Matthew Botvinic, Michael Mozer, and Danilo~Jimenez Rezende.
\newblock Learning to induce causal structure.
\newblock \emph{arXiv preprint arXiv:2204.04875}, 2022.

\bibitem[Kim et~al.(2021)Kim, Cho, Lee, and Hong]{kim2021multi}
Donggyun Kim, Seongwoong Cho, Wonkwang Lee, and Seunghoon Hong.
\newblock Multi-task neural processes.
\newblock \emph{arXiv preprint arXiv:2110.14953}, 2021.

\bibitem[Kim et~al.(2019)Kim, Mnih, Schwarz, Garnelo, Eslami, Rosenbaum, Vinyals, and Teh]{kim2019attentive}
Hyunjik Kim, Andriy Mnih, Jonathan Schwarz, Marta Garnelo, Ali Eslami, Dan Rosenbaum, Oriol Vinyals, and Yee~Whye Teh.
\newblock Attentive neural processes.
\newblock \emph{arXiv preprint arXiv:1901.05761}, 2019.

\bibitem[LeCun et~al.(1998)LeCun, Bottou, Bengio, and Haffner]{lecun1998gradient}
Yann LeCun, L{\'e}on Bottou, Yoshua Bengio, and Patrick Haffner.
\newblock Gradient-based learning applied to document recognition.
\newblock \emph{Proceedings of the IEEE}, 86\penalty0 (11):\penalty0 2278--2324, 1998.

\bibitem[Lee et~al.(2019)Lee, Lee, Kim, Kosiorek, Choi, and Teh]{lee2019set}
Juho Lee, Yoonho Lee, Jungtaek Kim, Adam Kosiorek, Seungjin Choi, and Yee~Whye Teh.
\newblock Set transformer: A framework for attention-based permutation-invariant neural networks.
\newblock In \emph{International conference on machine learning}, pages 3744--3753. PMLR, 2019.

\bibitem[Loshchilov and Hutter(2017)]{loshchilov2017decoupled}
Ilya Loshchilov and Frank Hutter.
\newblock Decoupled weight decay regularization.
\newblock \emph{arXiv preprint arXiv:1711.05101}, 2017.

\bibitem[Nguyen and Grover(2022)]{nguyen2022transformer}
Tung Nguyen and Aditya Grover.
\newblock Transformer neural processes: Uncertainty-aware meta learning via sequence modeling.
\newblock \emph{arXiv preprint arXiv:2207.04179}, 2022.

\bibitem[Shen et~al.(2024)Shen, Zhen, Wang, and Worring]{shen2024episodic}
Jiayi Shen, Xiantong Zhen, Qi~Wang, and Marcel Worring.
\newblock Episodic multi-task learning with heterogeneous neural processes.
\newblock \emph{Advances in Neural Information Processing Systems}, 36, 2024.

\bibitem[Vaswani et~al.(2017)Vaswani, Shazeer, Parmar, Uszkoreit, Jones, Gomez, Kaiser, and Polosukhin]{vaswani2017attention}
Ashish Vaswani, Noam Shazeer, Niki Parmar, Jakob Uszkoreit, Llion Jones, Aidan~N Gomez, {\L}ukasz Kaiser, and Illia Polosukhin.
\newblock Attention is all you need.
\newblock \emph{Advances in neural information processing systems}, 30, 2017.

\bibitem[Wagstaff et~al.(2022)Wagstaff, Fuchs, Engelcke, Osborne, and Posner]{wagstaff2022universal}
Edward Wagstaff, Fabian~B Fuchs, Martin Engelcke, Michael~A Osborne, and Ingmar Posner.
\newblock Universal approximation of functions on sets.
\newblock \emph{The Journal of Machine Learning Research}, 23\penalty0 (1):\penalty0 6762--6817, 2022.

\bibitem[Xu et~al.(2023)Xu, Zhu, and Clifton]{xu2023multimodal}
Peng Xu, Xiatian Zhu, and David~A Clifton.
\newblock Multimodal learning with transformers: A survey.
\newblock \emph{IEEE Transactions on Pattern Analysis and Machine Intelligence}, 2023.

\bibitem[Xu et~al.(2022)Xu, Zhao, Vineet, Lim, and Torralba]{xu2022mtformer}
Xiaogang Xu, Hengshuang Zhao, Vibhav Vineet, Ser-Nam Lim, and Antonio Torralba.
\newblock Mtformer: Multi-task learning via transformer and cross-task reasoning.
\newblock In \emph{European Conference on Computer Vision}, pages 304--321. Springer, 2022.

\bibitem[Zaheer et~al.(2017)Zaheer, Kottur, Ravanbakhsh, Poczos, Salakhutdinov, and Smola]{zaheer2017deep}
Manzil Zaheer, Satwik Kottur, Siamak Ravanbakhsh, Barnabas Poczos, Russ~R Salakhutdinov, and Alexander~J Smola.
\newblock Deep sets.
\newblock \emph{Advances in neural information processing systems}, 30, 2017.

\bibitem[Zhang and Yan(2022)]{zhang2022crossformer}
Yunhao Zhang and Junchi Yan.
\newblock Crossformer: Transformer utilizing cross-dimension dependency for multivariate time series forecasting.
\newblock In \emph{The eleventh international conference on learning representations}, 2022.

\end{thebibliography}

\appendix

\section{Proof of Theorem 1}
\label{app:icicl_proof}
In this section we give a detailed proof of \Cref{thm:icicl}, which we repeat here for completeness.

\addtocounter{theorem}{-1}
\begin{theorem}[In-context in-context learning]
    Let $\xi_i \sim p(\xi)$, $\mcD_i, \{\mcD_j\} \sim P(\xi_i)$. Let $p(\bfy | \bfx, \mcD_i, \xi_i)$ be the marginal posterior distribution of $P(\xi_i)$ given $\mcD_i$, $p(\bfy | \bfx, \mcD_i, \{\mcD_j\})$ be the marginal posterior distribution of the stochastic process $P$ given $\mcD_i$ and $\{\mcD_j\}$, and $p(\bfy | \bfx, \mcD_i)$ be the marginal posterior distribution of the stochastic process $P$ given $\mcD_i$. Then,
    \begin{equation*}
        \Exp{\mcD_i, \{\mcD_j\}, \xi_i}{\KL{p(\bfy | \bfx, \mcD_i, \xi_i)}{p(\bfy | \bfx, \mcD_i, \{\mcD_j\})}} \leq \Exp{\mcD_i, \xi_i}{\KL{p(\bfy | \bfx, \mcD_i, \xi_i)}{p(\bfy | \bfx, \mcD_i)}}.
    \end{equation*}
\end{theorem}
% \begin{equation}
%         \Exp{\mcD_i, \{\mcD_j\}, \xi_i}{\KL{p_{\xi_i}(\bfy | \bfx, \mcD_i)}{p(\bfy | \bfx, \mcD_i, \{\mcD_j\})}} \leq \Exp{\mcD_i, \xi_i}{\KL{p_{\xi_i}(\bfy | \bfx, \mcD_i)}{p(\bfy | \bfx, \mcD_i)}}.
%     \end{equation}
% where $\xi_i \sim p(\xi)$, $\mcD_i, \{\mcD_j\} \sim P(\xi_i)$, $p_{\xi_i}(\bfy | \bfx, \mcD_i)$ is the marginal posterior distribution of $P(\xi_i)$ given $\mcD_i$, $p(\bfy | \bfx, \mcD_i, \{\mcD_j\})$ is the marginal posterior distribution of the stochastic process $P$ given $\mcD_i$ and $\{\mcD_j\}$, and $p(\bfy | \bfx, \mcD_i)$ is the marginal posterior distribution of stochastic process $P$ given $\mcD_i$. 

\begin{proof}
First, observe that both terms of the inequality can be split into two parts:
\begin{equation}
\label{eq:left-side}
    \begin{aligned}
    & \Exp{\mcD_i, \{\mcD_j\}, \xi_i}{\KL{p(\bfy | \bfx, \mcD_i, \xi_i)}{p(\bfy | \bfx, \mcD_i, \{\mcD_j\})}} \\
    & = \Exp{\mcD_i, \{\mcD_j\}, \xi_i}{\int \underbrace{p(\bfy | \bfx, \mcD_i, \xi_i) \log p(\bfy | \bfx, \mcD_i, \xi_i)}_{\text{term 1}} d\bfy - \underbrace{p(\bfy | \bfx, \mcD_i, \xi_i) \log p(\bfy | \bfx, \mcD_i, \{\mcD_j\})}_{\text{term 2}} d\bfy}
    \end{aligned}
\end{equation}
and
\begin{equation}
\label{eq:right-side}
    \begin{aligned}
    & \Exp{\mcD_i, \xi_i}{\KL{p(\bfy | \bfx, \mcD_i, \xi_i)}{p(\bfy | \bfx, \mcD_i)}} \\
    & = \Exp{\mcD_i, \xi_i}{\int \underbrace{p(\bfy | \bfx, \mcD_i, \xi_i) \log p(\bfy | \bfx, \mcD_i, \xi_i)}_{\text{term 1}} d\bfy - \underbrace{p(\bfy | \bfx, \mcD_i, \xi_i) \log p(\bfy | \bfx, \mcD_i)}_{\text{term 2}} d\bfy}.
    \end{aligned}
\end{equation}
We first show that term 1 is the same for both sides of the inequality:
\begin{align}
\label{eq:first_term}
    \nonumber &\Exp{\mcD_i, \{\mcD_j\}, \xi_i}{\int p(\bfy | \bfx, \mcD_i, \xi_i) \log p(\bfy | \bfx, \mcD_i, \xi_i) d\bfy} \\
    \nonumber & = \Exp{ \xi_i \sim p(\xi), \mcD_i \sim P(\xi_i), \{\mcD_j\} \sim P(\xi_i)}{\int p(\bfy | \bfx, \mcD_i, \xi_i) \log p_(\bfy | \bfx, \mcD_i, \xi_i) d\bfy} \\
    \nonumber & = \Exp{ \xi_i \sim p(\xi), \mcD_i \sim P(\xi_i)}{\int p(\bfy | \bfx, \mcD_i, \xi_i) \log p_(\bfy | \bfx, \mcD_i, \xi_i) d\bfy} \\
    \nonumber &= \Exp{\mcD_i, \xi_i}{\int p(\bfy | \bfx, \mcD_i, \xi_i) \log p(\bfy | \bfx, \mcD_i, \xi_i) d\bfy} \\
    & = - \Exp{\mcD_i, \xi_i}{\entropy{p(\bfy | \bfx, \mcD_i, \xi_i)}}
\end{align}
where to go from the second to the third line we leveraged the fact that the term within the expectation does not depend on $\{\mcD_j\}$, and can, hence, be integrated out.
Thus, the elements of the inequality only differ in term 2. Consider term 2 in \Cref{eq:left-side}:
\begin{align*}
    &\Exp{\mcD_i, \{\mcD_j\}, \xi_i}{\int p(\bfy | \bfx, \mcD_i, \xi_i) \log p(\bfy | \bfx, \mcD_i, \{\mcD_j\}) d\bfy}  \\
    & =  \Exp{\mcD_i, \{\mcD_j\}}
    {\Exp{\xi_i \sim p(\xi | \mcD_i, \{ \mcD_j \})}{\int p(\bfy | \bfx, \mcD_i, \xi_i) \log p(\bfy | \bfx, \mcD_i, \{\mcD_j\}) d\bfy}}
\end{align*}
Note that $p(\bfy | \bfx, \mcD_i, \xi_i) = p(\bfy | \bfx, \mcD_i, \{\mcD_j\}, \xi_i)$---considering that $\mcD_i$ and $\{\mcD_j\}$ are conditionally independent given $\xi_i$. Informally, this means that additionally conditioning on the in-context datasets $\{\mcD_j\}$ does not give us any extra information, given that we are already considering the marginal posterior of $P(\xi_i)$. Hence,
\begin{align}
\label{eq:second_term1}
    \nonumber &\Exp{\mcD_i, \{\mcD_j\}, \xi_i}{\int p(\bfy | \bfx, \mcD_i, \xi_i) \log p(\bfy | \bfx, \mcD_i, \{\mcD_j\}) d\bfy} \\
    \nonumber & = \Exp{\mcD_i, \{\mcD_j\}}
    {\Exp{\xi_i \sim p(\xi | \mcD_i, \{\mcD_j\})}{\int p(\bfy | \bfx, \mcD_i, \{\mcD_j\}, \xi_i) \log p(\bfy | \bfx, \mcD_i, \{\mcD_j\}) d\bfy}} \\
    \nonumber & = \Exp{\mcD_i, \{\mcD_j\}}
    {\int p(\bfy | \bfx, \mcD_i, \{\mcD_j\}) \log p(\bfy | \bfx, \mcD_i, \{\mcD_j\}) d\bfy}\\
    \nonumber & = - \Exp{\mcD_i, \{\mcD_j\}}{\entropy{p(\bfy | \bfx, \mcD_i, \{\mcD_j\})}} \\
    & = - \Exp{\mcD_i}{\Exp{\{\mcD_j\} | \mcD_i}{\entropy{p(\bfy | \bfx, \mcD_i, \{\mcD_j\})}}}
\end{align}
where we used Fubini's theorem to go from the second to the third line. Similarly, for term 2 of \Cref{eq:right-side} we have:
\begin{align}
\label{eq:second_term2}
    \nonumber &\Exp{\mcD_i, \xi_i}{\int p(\bfy | \bfx, \mcD_i, \xi_i) \log p(\bfy | \bfx, \mcD_i) d\bfy} \\
    \nonumber & = \Exp{\mcD_i}
    {\Exp{\xi_i \sim p(\xi | \mcD_i)}{\int p(\bfy | \bfx, \mcD_i, \xi_i) \log p(\bfy | \bfx, \mcD_i) d\bfy}} \\
    \nonumber & = \Exp{\mcD_i}
    {\int p(\bfy | \bfx, \mcD_i) \log p(\bfy | \bfx, \mcD_i) d\bfy}\\
    & = - \Exp{\mcD_i}{\entropy{p(\bfy | \bfx, \mcD_i)}}.
\end{align}
Putting the results from \Cref{eq:first_term} and \Cref{eq:second_term1} together we obtain
\begin{align}
    &\Exp{\mcD_i, \{\mcD_j\}, \xi_i}{\KL{p(\bfy | \bfx, \mcD_i, \{\mcD_j\}, \xi_i)}{p(\bfy | \bfx, \mcD_i)}} \\
    & = - \Exp{\mcD_i, \xi_i}{\entropy{p(\bfy | \bfx, \mcD_i, \xi_i)}} + \Exp{\mcD_i}{\Exp{\{\mcD_j\} | \mcD_i}{\entropy{p(\bfy | \bfx, \mcD_i, \{\mcD_j\})}}},
\end{align}
and putting \Cref{eq:first_term} and \Cref{eq:second_term2} together we have
\begin{align}
    &\Exp{\mcD_i, \xi_i}{\KL{p(\bfy | \bfx, \mcD_i, \xi_i)}{p(\bfy | \bfx, \mcD_i)}} \\
    & = - \Exp{\mcD_i, \xi_i}{\entropy{p(\bfy | \bfx, \mcD_i, \xi_i)}} + \Exp{\mcD_i}{\entropy{p(\bfy | \bfx, \mcD_i)}}.
\end{align}
Since conditioning reduces entropy $\entropy{p(\bfy | \bfx, \mcD_i, \{\mcD_j\}} \leq \entropy{p(\bfy | \bfx, \mcD_i)} \ \forall \{\mcD_j\} \sim P(\xi_i)$, we have
\begin{equation*}
\resizebox{\linewidth}{!}{
$\displaystyle
\Exp{\mcD_i}{\Exp{\{\mcD_j\} | \mcD_i}{\entropy{p(\bfy | \bfx, \mcD_i, \{\mcD_j\})}}} \leq \Exp{\mcD_i}{\Exp{\{\mcD_j\} | \mcD_i}{\entropy{p(\bfy | \bfx, \mcD_i)}}} = \Exp{\mcD_i}{\entropy{p(\bfy | \bfx, \mcD_i)}},
$
}
\end{equation*}
which leads to the result in \Cref{thm:icicl}.
\end{proof}

\section{Defining the MHSA and MHCA Operations}
\label{app:mhsa_mhca}
Let $\bfZ^{\ell} \in \R^{N\times D_z}$ denote the input set to the $\ell$-th MHSA operation. The MHSA operation updates the $n$\textsuperscript{th} token $\bfz^{\ell}_n$ as
\begin{equation}\textstyle
\label{eq:self-attention}
    \!\!\tilde{\bfz}^{\ell}_n \!=\! \operatorname{cat}\!\Big(\Big\{\sum_{m=1}^N \alpha^{\ell}_h(\bfz^{\ell}_n, \bfz^{\ell}_m) {\bfz^{\ell}_m\!}^T \bfW^{\ell}_{V, h}\Big\}_{h=1}^{H^{\ell}}\Big)\bfW^{\ell}_O
\end{equation}
where $\operatorname{cat}$ denotes the concatenation operation across the last dimension. Here, $\bfW^{\ell}_{V, h} \in \R^{D_z\times D_V}$ and $\bfW^{\ell}_{O} \in \R^{H^{\ell}D_V \times D_z}$ are the value and projection weight matrices, where $H^{\ell}$ denotes the number of `heads' in layer $\ell$. Note that permutation equivariance is achieved through the permutation invariant summation operator. As this is the only mechanism through which the tokens interact with each other, permutation equivariance for the overall model is ensured. The attention mechanism, $\alpha^{\ell}_h$, is implemented as
\begin{equation}
\label{eq:attention-mechanism}
    \alpha^{\ell}_h(\bfz^{\ell}_n, \bfz^{\ell}_m) = \frac{e^{{\bfz^{\ell}_n}^T\bfW^{\ell}_{Q, h}\left[\bfW^{\ell}_{K, h}\right]^T\bfz^{\ell}_m}}{\sum_{m=1}^N e^{{\bfz^{\ell}_n}^T\bfW^{\ell}_{Q, h}\left[\bfW^{\ell}_{K, h}\right]^T\bfz^{\ell}_m}}
\end{equation}
where $\bfW^{\ell}_{Q, h} \in \R^{D_z \times D_{QK}}$ and $\bfW^{\ell}_{K, h} \in \R^{D_z\times D_{QK}}$ are the query and key weight matrices. The softmax-normalisation ensures that $\sum_{m=1}^N \alpha^{\ell}_h(\bfz^{\ell}_n, \bfz^{\ell}_m) = 1 \ \forall n, h, \ell$.

Often, conditional independencies amongst the set of tokens---in the sense that the set $\{\bfz^{\ell}_n\}^{\ell=L}_{\ell=1}$ do not depend on the set $\{\bfz^{\ell}_m\}^{\ell=L}_{\ell=0}$ for some $n,\ m \in \{1, \ldots, N\}$---are desirable. This is typically achieved through masking, such that the pre-softmax activations are replaced by $\tilde{\alpha}^{\ell}_h$, where 
\begin{equation}
    \tilde{\alpha}^{\ell}_h(\bfz^{\ell}_n, \bfz^{\ell}_m) = \begin{cases}
        -\infty, & \text{$m\in A(n)$.} \\
        {\bfz^{\ell}_n}^T\bfW^{\ell}_{Q, h}\left[\bfW^{\ell}_{K, h}\right]^T\bfz^{\ell}_m, & \text{otherwise.}
    \end{cases}
\end{equation}
Here, $A(n) \subseteq \N_{\leq N}$ indexes the set of tokens we wish to make the update for token $\bfz^{\ell}_n$ independent of. If $A(n) = A$ (i.e.\ the same set of tokens are conditioned on for every $n$) then in practice it is more computationally efficient to use MHCA operations together with MHSA operations than it is to directly compute \Cref{eq:self-attention}. An MHCA operation uses the subset of tokens $\{\bfz^{\ell}_m | m\in A\}$ to update the complementary set of tokens $\{\bfz^{\ell}_n | n\in A^c\}$ in a computationally efficient manner:
\begin{equation}\textstyle
\label{eq:cross-attention}
    \!\!\tilde{\bfz}^{\ell}_n \!=\! \operatorname{cat}\!\Big(\Big\{\sum_{m\in A} \alpha^{\ell}_h(\bfz^{\ell}_n, \bfz^{\ell}_m) {\bfz^{\ell}_m\!}^T \bfW^{\ell}_{V, h}\Big\}_{h=1}^{H^{\ell}}\Big)\bfW^{\ell}_O \qquad \forall n \in A^c.
\end{equation}
For $N$ tokens that solely depend on a subset of $N_1$ tokens, the computational complexity is reduced from $\order{N^2}$ using masked-\gls{mhsa} operations to $\order{NN_1}$ using \gls{mhca} operations.

\section{Some \gls{tnp} Architectures}
\label{app:tnp_architectures}
There exist a number of architectures used in different members of the \gls{tnp}. We provide diagramatic illustrations of the following: the \gls{anp} of \cite{kim2019attentive} in \Cref{fig:anp}; the \gls{tnp} of \cite{nguyen2022transformer} in \Cref{fig:tnp}; the \gls{lbanp} of \cite{feng2022latent} in \Cref{fig:lbanp}; and the \gls{ist}-style \citep{lee2019set} \gls{tnp} in \Cref{fig:ist}.

\begin{figure}[htb]
    \centering
    \includegraphics[width=0.75\textwidth]{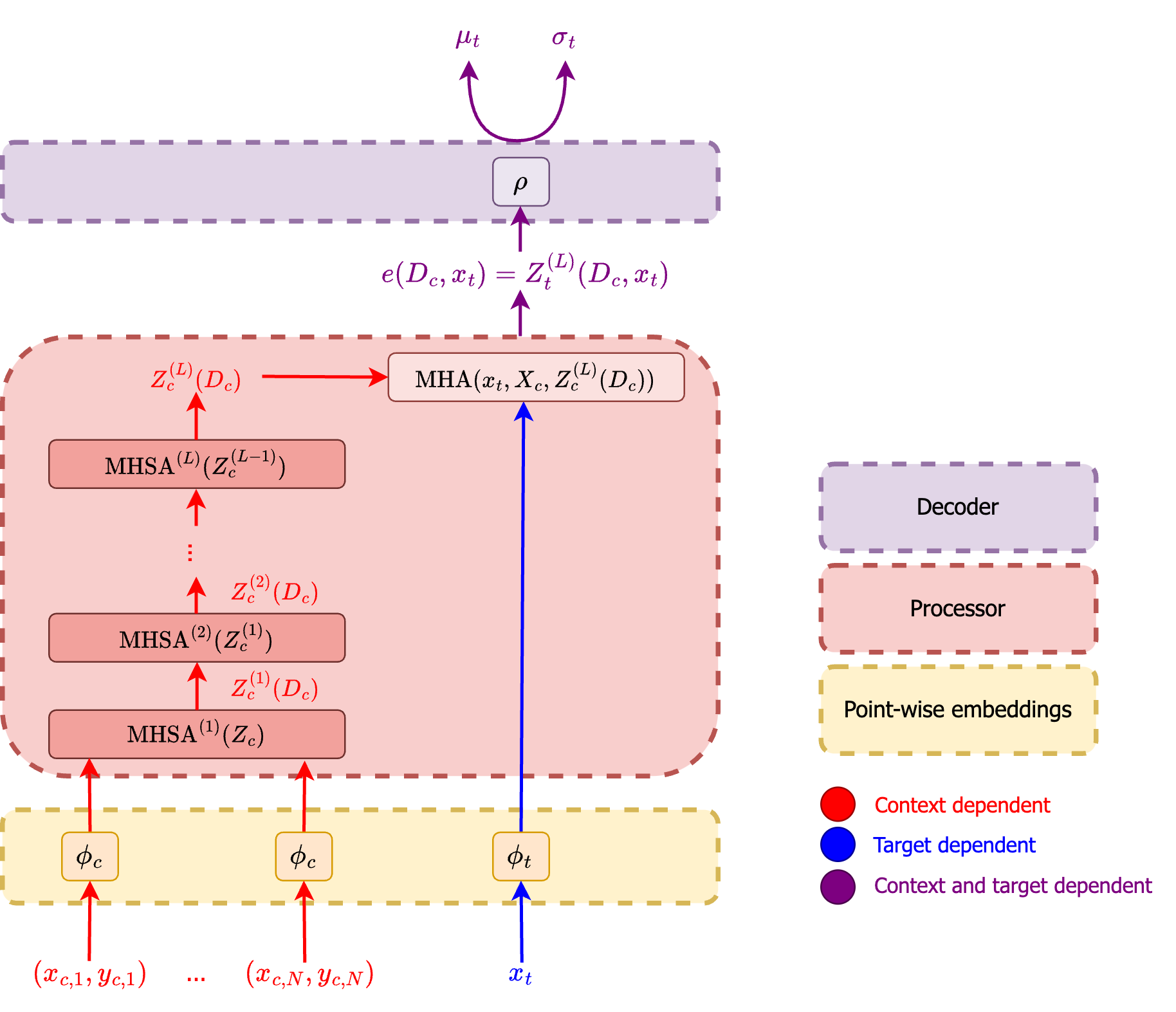}
    \caption{A diagram illustrating the architecture of the \gls{anp} \citep{kim2019attentive}.}
    \label{fig:anp}
\end{figure}

\begin{figure}[htb]
    \centering
    \includegraphics[width=0.75\textwidth]{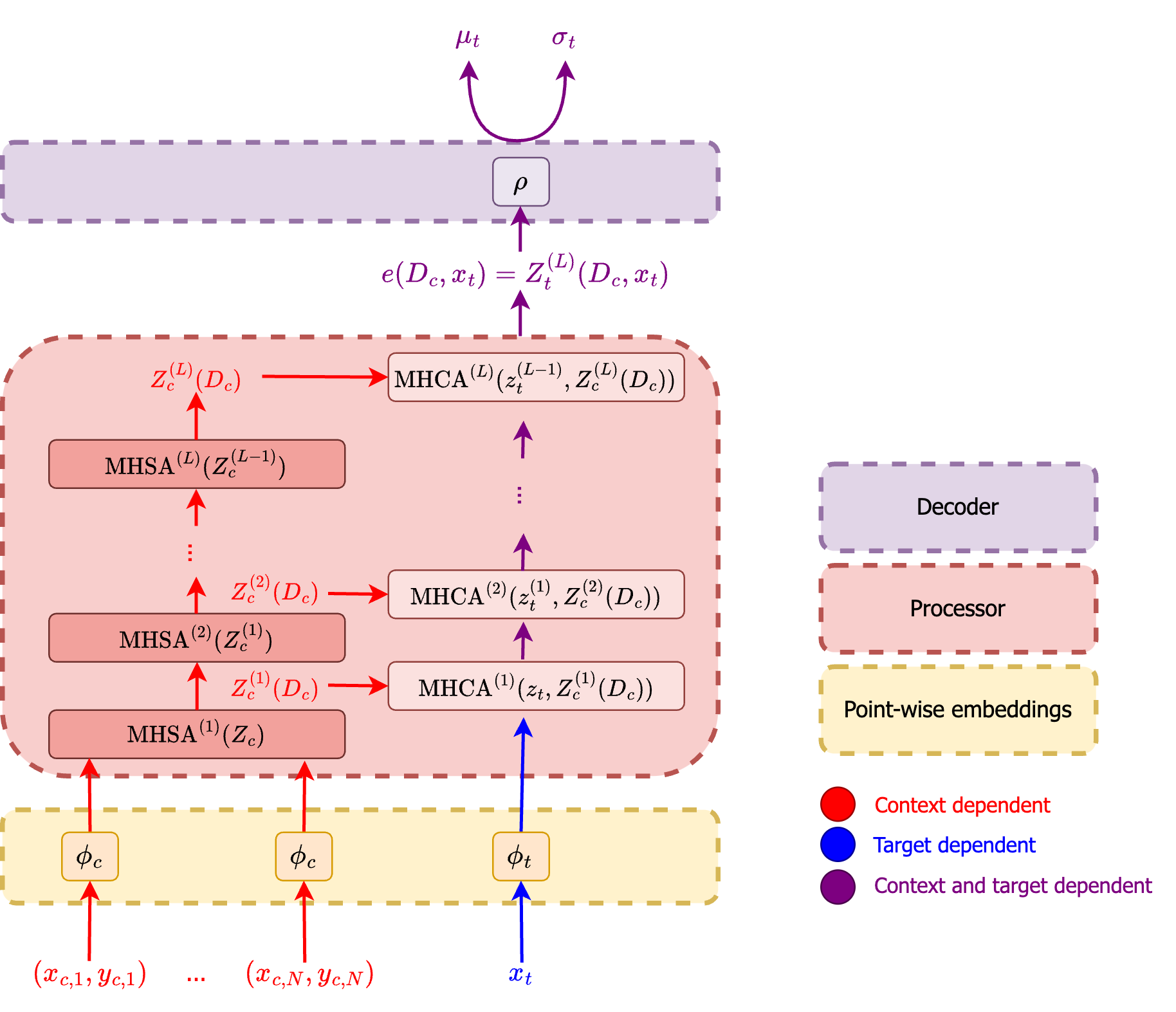}
    \caption{A diagram illustrating the architecture of the \gls{tnp} \citep{nguyen2022transformer}.}
    \label{fig:tnp}
\end{figure}

\begin{figure}[htb]
    \centering
    \includegraphics[width=0.75\textwidth]{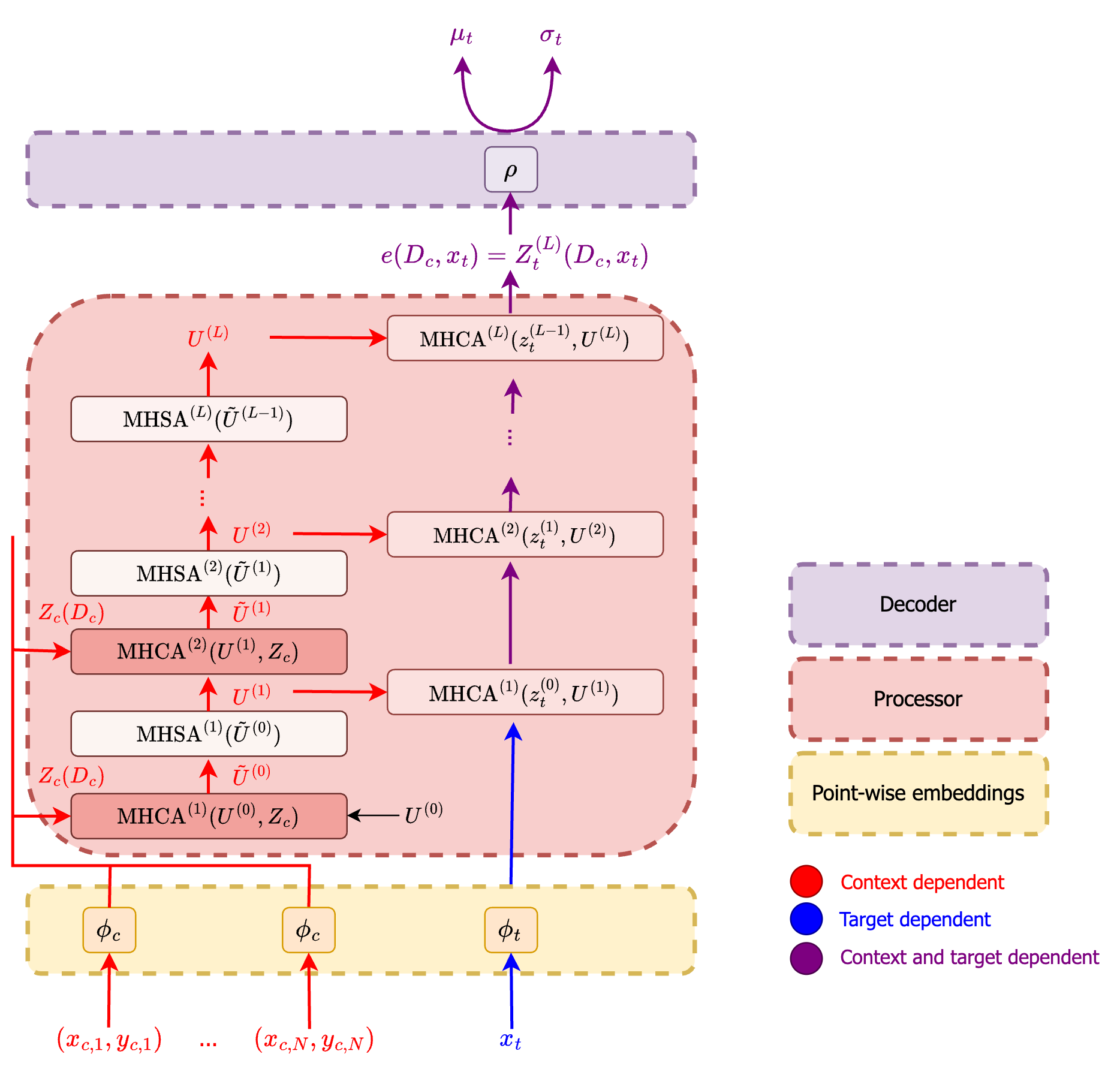}
    \caption{A diagram illustrating the architecture of the \gls{lbanp} \citep{feng2022latent}.}
    \label{fig:lbanp}
\end{figure}

\begin{figure}[htb]
    \centering
    \includegraphics[width=\textwidth]{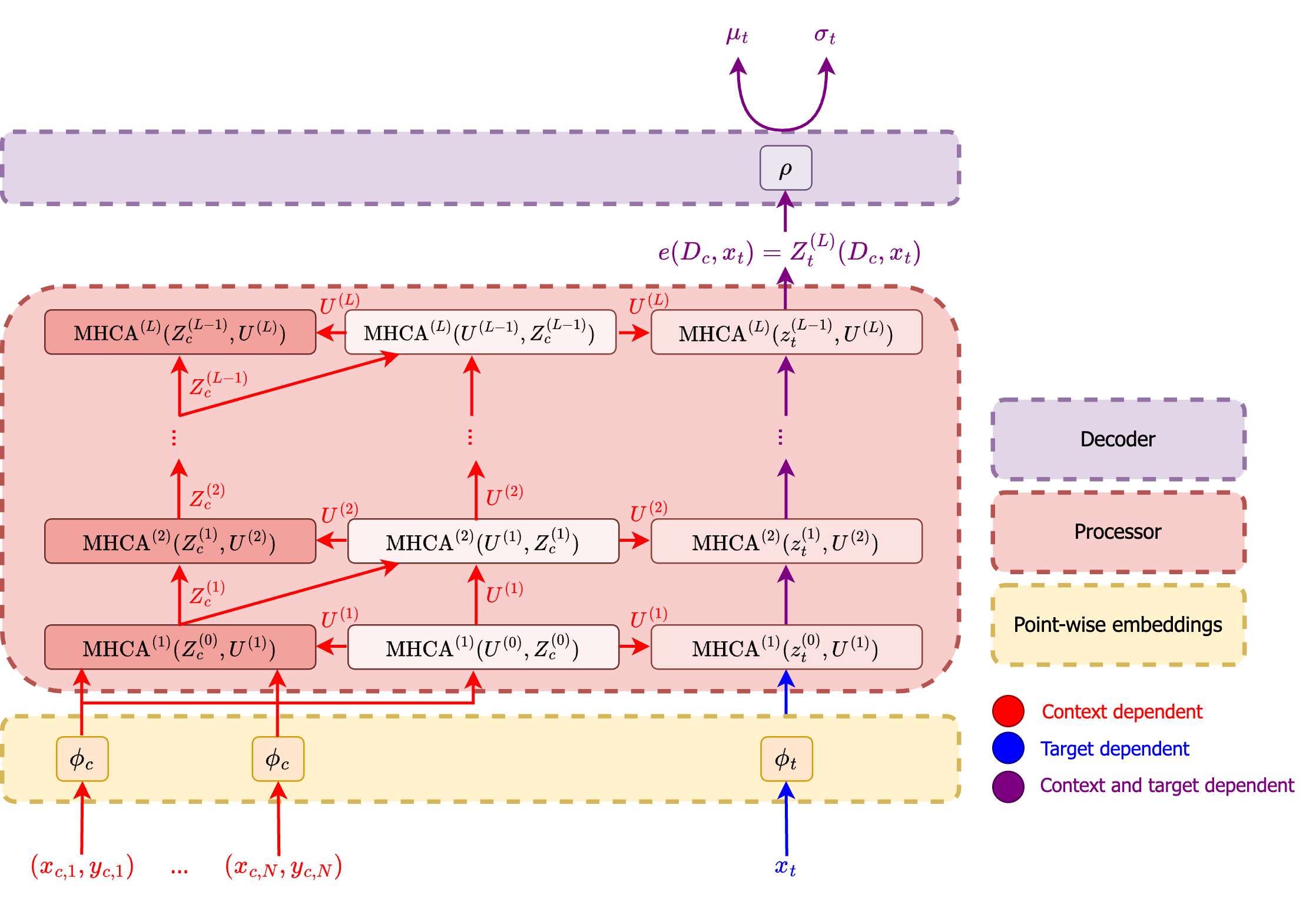}
    \caption{A diagram illustrating the architecture of the \gls{ist}-style \gls{tnp} \citep{lee2019set}.}
    \label{fig:ist}
\end{figure}

\section{Set Functions and Set of Sets Functions}
\label{app:valid_set_function}
In this section we utilise the Deepset result of \cite{zaheer2017deep}, which states that any permutation invariant function on the set $\{z_n\}_{n=1}^N$ can be expressed as
\begin{equation}
    f(z_1, z_2, \ldots, z_N) = \rho\left(\sum_{n=1}^N \phi(z_n)\right)
\end{equation}
for continuous functions $\rho$ and $\phi$. As shown in \citep{lee2019set}, the Deepset formulation is a special case of pseudo-token based transformers operating on sets. It is therefore sufficient to consider only this Deepset formulation.

Let $N_{ic}$ be the number of in-context datasets, each with $N_{1}, N_{2}, \dots, N_{N_{ic}}$ datapoints. Each in-context dataset is embedded into a latent representation through a valid set function of the form
$$f(z_{n, 1}, z_{n, 2}, \dots, z_{n, N_n}) = \rho_{n}\bigg(\sum_{j=1}^{N_n}\phi_{n}(z_{n, j})\bigg)$$
for all $n \in \{1, 2, \dots, N_{ic}\}$.
When combining the representation of all of the in-context datasets, we also obtain a valid set function
\begin{align*}
    & f\Bigg(\{z_{1, 1}, z_{1, 2}, \dots, z_{1, N_1}\}, \dots, \{z_{N_{ic}, 1}, z_{N_{ic}, 2}, \dots, z_{N_{ic}, N_{N_{ic}}}\}\Bigg) = \\
    & f\Bigg(\rho_1\bigg(\sum_{j=1}^{N_1}\phi_1(z_{1, j})\bigg), \dots, \rho_{N_{ic}}\bigg(\sum_{j=1}^{N_{N_{ic}}}\phi_{N_{ic}}(z_{N_{ic}, j})\bigg)\Bigg) = \\
    & \rho\Bigg( \sum_{n=1}^{N_{ic}} \phi \bigg( \rho_n\big(\sum_{j=1}^{N_n}\phi_n(z_{n, j})\big)\bigg)\Bigg).
\end{align*}
Thus, the latent representation of the in-context datasets is a valid set function on the set of in-context datapoints.

\section{Other Choices of Architecture for the \gls{model}}
\label{app:icicl_tnp}
As discussed in \Cref{sec:icicl_tnp}, there exist a number of possible ways to construct the \gls{model}. In \Cref{fig:ic-lbanp}, we provide an alternative architecture for the \gls{model} than that shown in \Cref{fig:icicl-tnp}. We also evaluated the performance of this architecture, and generally found little difference in performance.

\begin{figure}[htb]
    \centering
    \includegraphics[width=\textwidth]{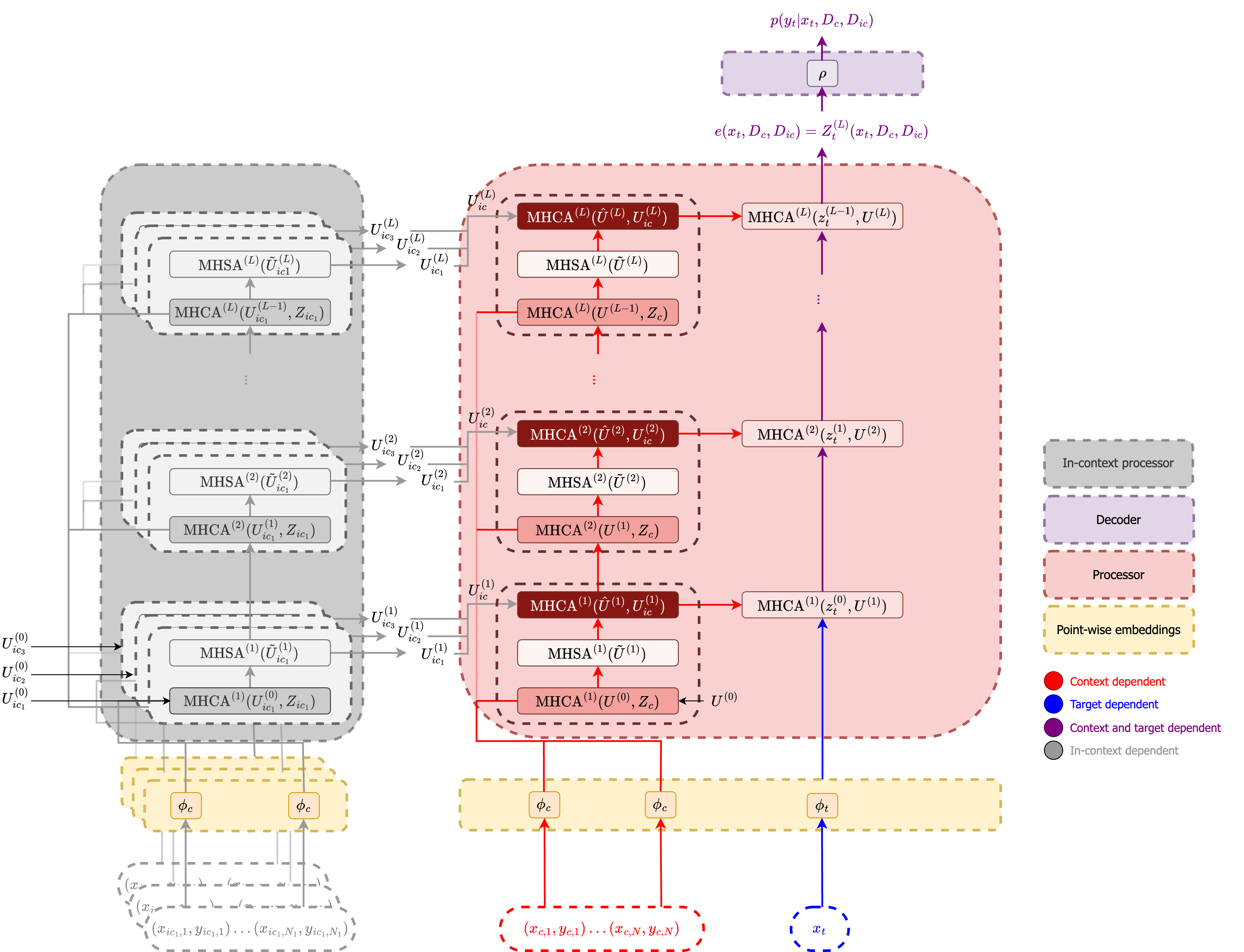}
    \caption{A diagram illustrating the a possible choice for the architecture of the \gls{model}.}
    \label{fig:ic-lbanp}
\end{figure}

\section{ICICL-CNP}
\label{app:icicl_cnp}
It is possible to extend the capabilities of the \gls{cnp} to perform in-context in-context learning. In addition to obtaining a latent representation of the context dataset, we obtain latent representations of the in-context datasets in an identical manner (i.e.\ using a Deepset \citep{zaheer2017deep}). These are aggregated, and the result is then aggregated with the context latent representation to get a latent representation of both the context dataset and in-context datasets. This is then aggregated with the latent representation of the target input, which is then fed into the decoder. We provide a diagrammatic illustration of this in \Cref{fig:icicl_cnp}. 

\begin{figure}
    \centering
    \includegraphics[width=\textwidth]{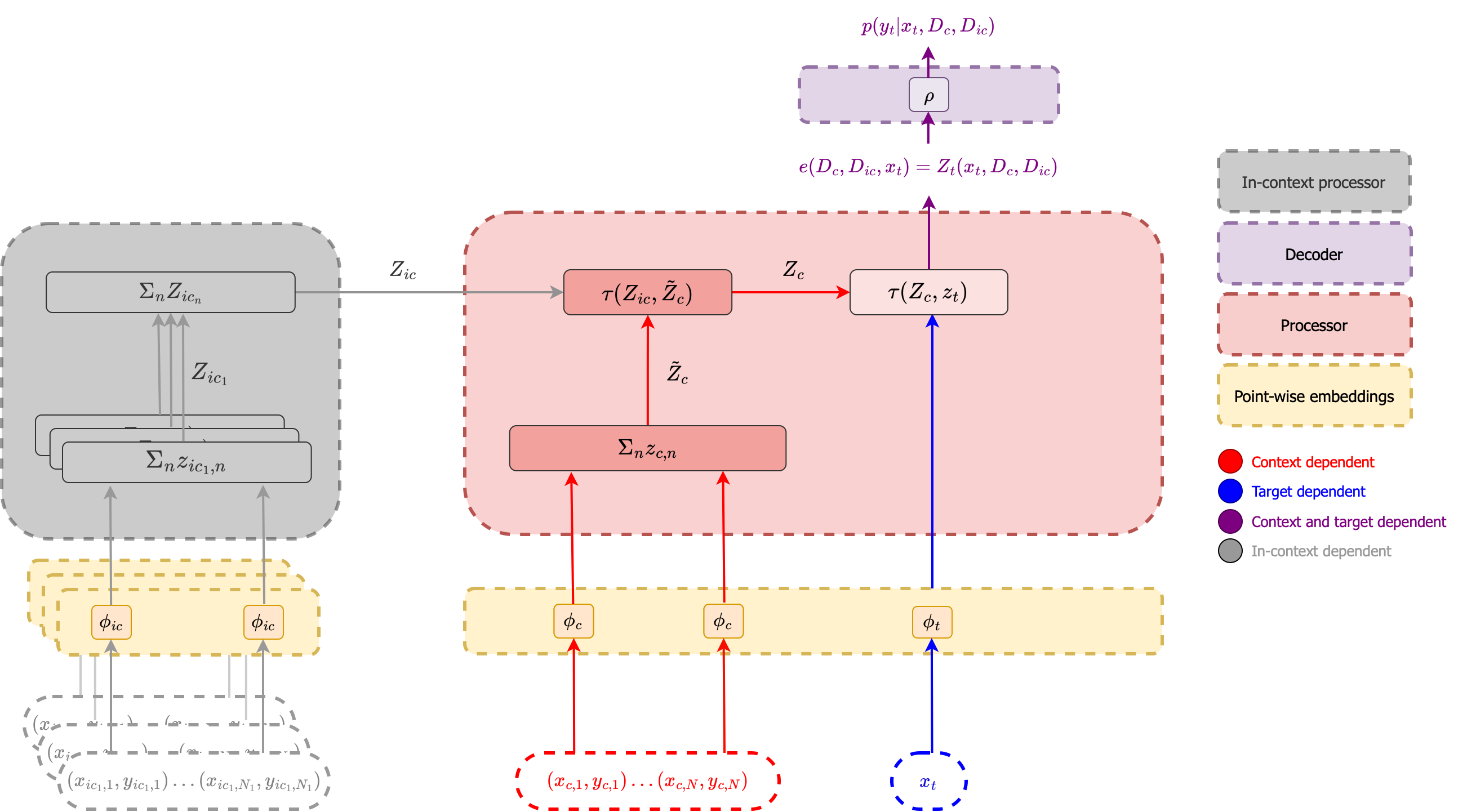}
    \caption{A diagram illustrating the ICICL-\gls{cnp} architecture with three in-context datasets. The point-wise embedding layer is used to get an initial representation of all datapoints, including the target input location $\bfx_t$. Then, for each dataset, its latent representation is obtained using the standard Deepset approach \citep{zaheer2017deep}. The latent representations for the in-context datasets are combined, and then aggregated with the latent representation for the context dataset. The resultant latent representation is aggregated with the target representation, and the processor outputs the encoder representation $e(\bfx_t, \mcD_c, \{\mcD_{ic, j}\}_{j=1}^{j=N_{ic}})$.}
    \label{fig:icicl_cnp}
\end{figure}

\section{Experiment Details}
\label{app:experiment_details}
For all models, we use an embedding / token size of $D_z = 128$, and point-wise encoder and decoder consisting of an MLP with two hidden layers of dimension $D_z$. The decoder parameterises the mean and pre-softplus variance of a Gaussian likelihood with heterogeneous noise. 

All transformer-based models use five layers of operations, with both the \gls{mhca} and \gls{mhsa} layers using $H = 8$ heads of dimension $D_V = D_{QK} = 16$. In each of the attention blocks, we apply a residual connection consisting of layer-normalisation to the input tokens followed by the attention mechanism. Following this, there is another residual connection consisting of layer-normalisation followed by a point-wise MLP with two hidden layers of dimension $D_z$. Initial pseudo-token values are sampled from a standard normal distribution. 

All \gls{cnp} models use Deepsets consisting of MLPs with five layers of dimension $D_z$ for the point-wise embedding. The latent representations for each dataset are obtained by aggregating (mean) the point-wise embeddings for all the datapoints contained in that dataset. The in-context latent representation is obtained by aggregating (mean) the latent representations for each in-context dataset. The in-context latent representation and context latent representation are aggregated through concatenation, being then concatenated again with the representation for the target location. This is then passed through a decoder, consisting of two hidden layers of dimension $D_z$.

We optimise the model parameters using AdamW \citep{loshchilov2017decoupled} with a learning rate of $5\times 10^{-4}$ and batch size of 16. We apply clipping to gradients with magnitudes greater than 0.5.

\subsection{Synthetic 1-D Regression}
\label{app:synthetic-1d}
For each dataset and in-context datasets, we first randomly sample a kernel between RBF and periodic. Then, we sample the kernel's hyperparameter $\ell$ - length-scale in the case of the RBF kernel and period in the case of the periodic kernel. This is sampled according to $\log \ell \sim \mcU_{[\log 0.25, \log 4]}$. In terms of the task, we sample  the number of context points $N_c \sim \mcU\{1, 64\}$, the number of in-context datasets $N_{ic} \sim \mcU\{0, 5\}$, the number of points in each in-context dataset $N_{ic, c} \sim \mcU\{64, 128\}$, the context inputs $x_{c, n} \sim \mcU_{[-2, 2]}$, the target inputs $x_{t, n} \sim \mcU_{[-4, 4]}$, and the inputs for the in-context datasets $x_{ic, j, n} \sim \mcU_{[-4, 4]}$. All tasks use the same number of target points $N_t = 128$. The observations for each task are drawn from a GP with kernel
\begin{equation}
    k_{\text{obs}} = k + \sigma_n^2 \delta(x - x')
\end{equation}
where the observation noise $\sigma_n = 0.2$.

For the \gls{model}, we use $M = 32$ context pseudo-tokens and $M_{ic} = 32$ in-context pseudo-tokens for each in-context dataset. For the \gls{pt-tnp}, we use $M = 32$ pseudo-tokens. Each model is trained for 1,000 epochs, with each epoch consisting of 1,000 iterations. We evaluate the performance of each model on 80,000 test datasets.

\subsubsection{Additional Results}
\label{subsubsec:gp-additional-results}
We also study the behaviour of the \gls{model} when conditioning on in-context datasets that are drawn from a different stochastic process as compared to the context dataset. In particular, in the example from \cref{fig:lbanp-ic-gp}, the GP the context datapoints are drawn from has a periodic kernel with $\ell = 1.85$. In this experiment, we investigate the predictive distribution in five scenarios:
\begin{enumerate}
    \item No in-context datasets;
    \item Three in-context datasets that are drawn from the same stochastic process as the context datapoints;
    \item Three in-context datasets where the datapoints are generated from a GP with an RBF, instead of a periodic kernel (with $\ell = 0.52$);
    \item Three in-context datasets where the datapoints are generated from a GP with a periodic kernel and a similar (but not identical) period $\ell = 1.76$;
    \item Three in-context datasets where the datapoints are generated from a GP with a periodic kernel and a significantly different period $\ell = 0.28$.
\end{enumerate}

The resulting predictive distributions are shown in \Cref{fig:gp-prediction-comparison}. As already noted in the main text, when conditioning on in-context data drawn from the same stochastic process, as opposed to no in-context information, the predictions improve significantly and the uncertainty decreases. When the in-context and context datapoints are drawn from a GP with a different kernel (RBF vs. periodic), the predictions of the \gls{model} are more uncertain, especially far away from the data, but the model still manages to fit the data well in regions with high context datapoints concentration (\Cref{fig:gp-RBF}). When the in-context and context datapoints are drawn from a GP with the same type of kernel and a similar period, the predictions improve significantly as compared to when no in-context information is available (\Cref{fig:gp-periodic-close}). This suggests that even if the two stochastic processes are not exactly the same, but with similar characteristics, the model is still able to extract useful information from the in-context data. In contrast, when the two stochastic processes have wildly different characteristic (i.e. period), the model is unable to explain the data (\Cref{fig:gp-periodic-far}), leading to a large increase in uncertainty in all regions. We note that this could be used as a way to identify whether the in-context and context datasets are indeed generated by the same stochastic process, with the \gls{model} outputting uncertain predictions when there is a severe discrepancy.

\begin{figure}[htbp]
\centering
    % \floatconts
    % {fig:gp-additional-comparison}
      {%
        \subfigure[No in-context datasets][c]{\label{fig:gp-noIC}%
          \includegraphics[width=0.65\linewidth,trim={0.5cm 0.5cm 0.5cm 0.5cm},clip]{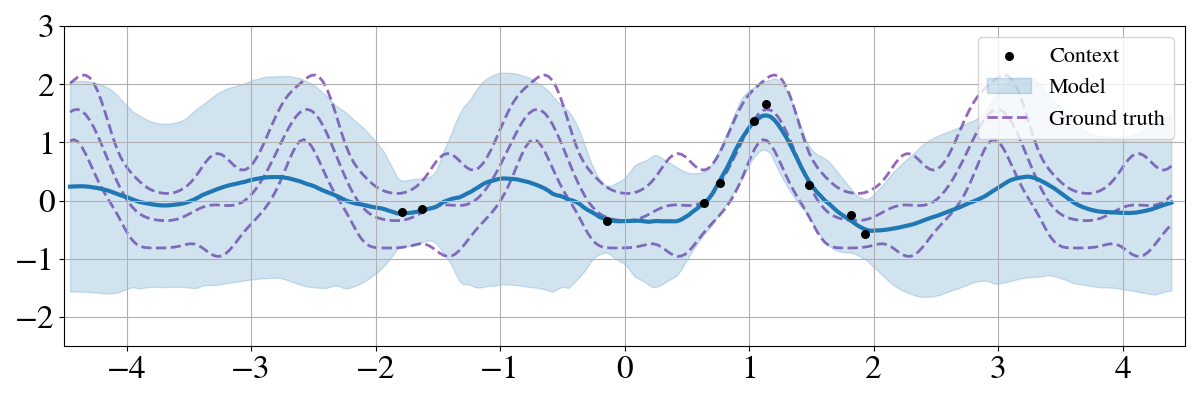}}%
        \\
        \subfigure[Periodic kernel, $\ell = 1.85$][c]{\label{fig:gp-3IC}%
          \includegraphics[width=0.65\linewidth,trim={0.5cm 0.5cm 0.5cm 0.5cm},clip]{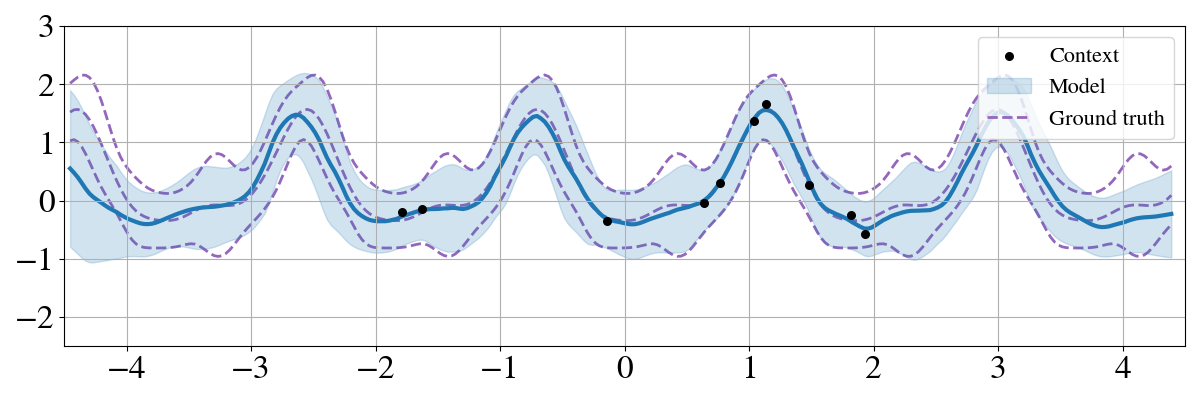}}%
        \\
        \subfigure[RBF kernel, $\ell = 0.52$][c]{\label{fig:gp-RBF}%
          \includegraphics[width=0.65\linewidth,trim={0.5cm 0.5cm 0.5cm 0.5cm},clip]{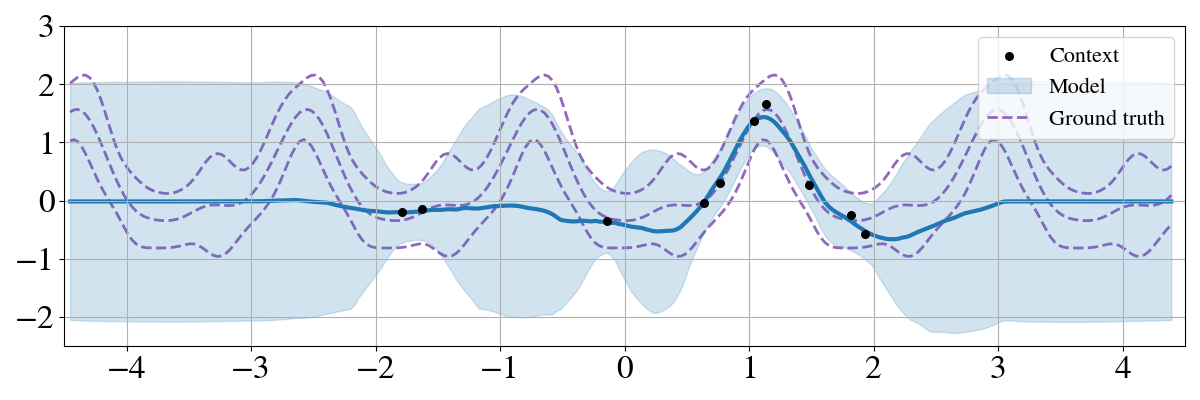}}%
        \\
        \subfigure[Periodic kernel, $\ell = 1.76$][c]{\label{fig:gp-periodic-close}%
          \includegraphics[width=0.65\linewidth,trim={0.5cm 0.5cm 0.5cm 0.5cm},clip]{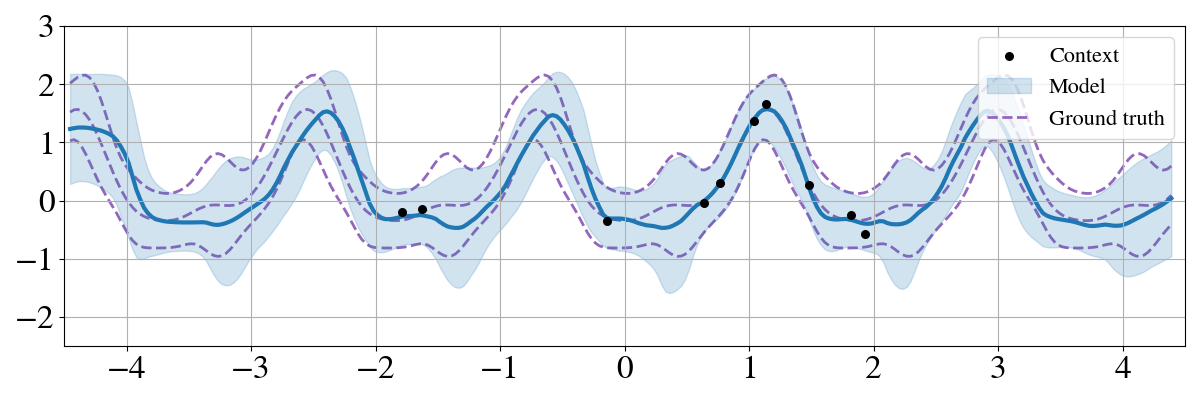}}%
        \\
        \subfigure[Periodic kernel, $\ell = 0.28$][c]{\label{fig:gp-periodic-far}%
          \includegraphics[width=0.65\linewidth,trim={0.5cm 0.5cm 0.5cm 0.5cm},clip]{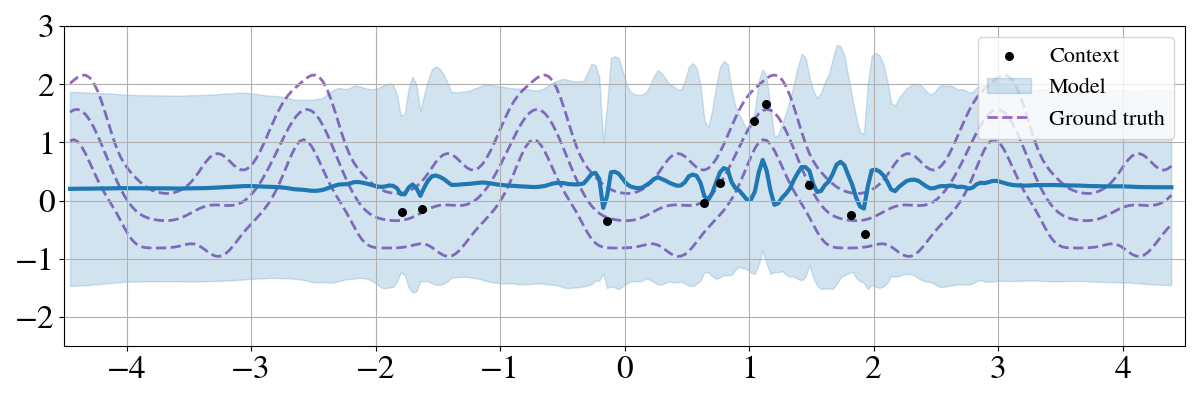}}%
      }
      {\caption{A comparison between the predictive distribution of the \gls{model} with different in-context conditioning information. The GP kernel used to generate the context datapoints was periodic with $\ell = 1.85$.}
      \label{fig:gp-prediction-comparison}}
\end{figure}

\subsubsection{Results for alternative architecture}
We also provide the results on this task for the alternative architecture shown in \Cref{fig:ic-lbanp}. We did not observe a significant difference between this model (\gls{model}2) and \gls{model} on this task, as illustrated in \Cref{table:ic-gp-pt-tnp2}.

\begin{table}[htbp]
\centering
\begin{tabular}{ccc}
\hline
$\boldsymbol{N_{ic}}$ & \textbf{\gls{model}} & \textbf{\gls{model}2} \\ \hline
0     & $-0.607  \pm 0.005 $ & $-0.612 \pm 0.005 $       \\
1     & $-0.499 \pm 0.005 $ & $-0.491 \pm 0.005 $       \\
2     & $-0.474  \pm 0.005 $  & $-0.473 \pm 0.005 $        \\
3     & $-0.469 \pm 0.005 $  & $-0.467 \pm 0.005 $      \\
4     & $-0.467  \pm 0.005 $  & $-0.464 \pm 0.005 $      \\
5     & $-0.466  \pm 0.005 $  & $-0.463 \pm 0.005 $      \\ \hline
\end{tabular}
\caption{Comparison of test log likelihood for the \gls{model} and the \gls{model}2 architectures on the synthetic GP task for a varying number of in-context datasets ($N_{ic}$). The two models have similar performance (within one standard deviation).}
\label{table:ic-gp-pt-tnp2}
\end{table}

\subsubsection{Results for OOD testing}
\label{subsubsec:gp-OOD}
We present some additional examples of predictions on samples that are drawn outside of the distribution the models have been trained on. In the top row of \Cref{fig:gp-prediction-OOD} we consider samples drawn from a GP with RBF kernel and $\ell = 5.19$, whereas in the bottom row we consider the challenging task of modelling datapoints drawn from a GP with a periodic kernel and $\ell = 0.24$. In both cases the \gls{model} is conditioned on three in-context datasets, containing datapoints satisfying the same stochastic process as the context data. The \gls{model} manages to leverage in-context learning to improve its predictions even when the samples are drawn from a different distribution than the one the model has been trained on.

\begin{figure}[htbp]
\centering
    % \floatconts
    % {fig:gp-additional-comparison}
      {%
        \subfigure[\gls{pt-tnp}, RBF $\ell=5.19$][c]{\label{fig:gp-tnp-RBF-OOD}%
          \includegraphics[width=0.49\linewidth,trim={0.5cm 0.5cm 0.5cm 0.5cm},clip]{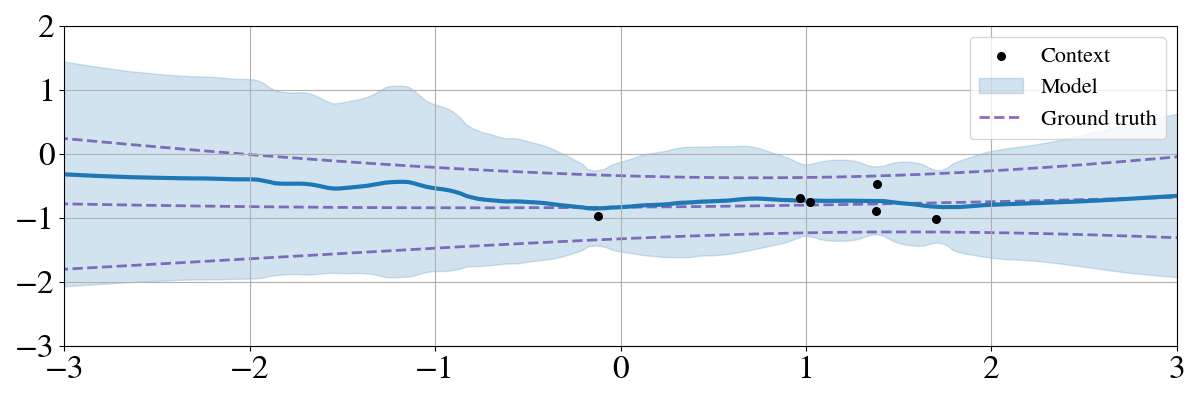}}%
        \hfill
        \subfigure[\gls{model}, RBF $\ell=5.19$][c]{\label{fig:gp-icicltnp-rbf-OOD}%
          \includegraphics[width=0.49\linewidth,trim={0.5cm 0.5cm 0.5cm 0.5cm},clip]{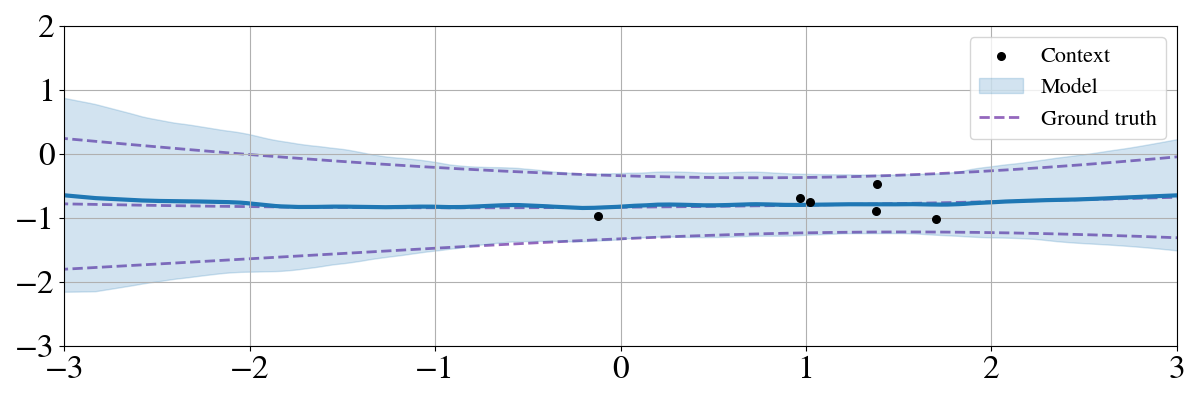}}%
        \hfill
        \subfigure[\gls{pt-tnp}, Periodic $\ell = 0.24$][c]{\label{fig:gp-tnp-periodic-OOD}%
          \includegraphics[width=0.49\linewidth,trim={0.5cm 0.5cm 0.5cm 0.5cm},clip]{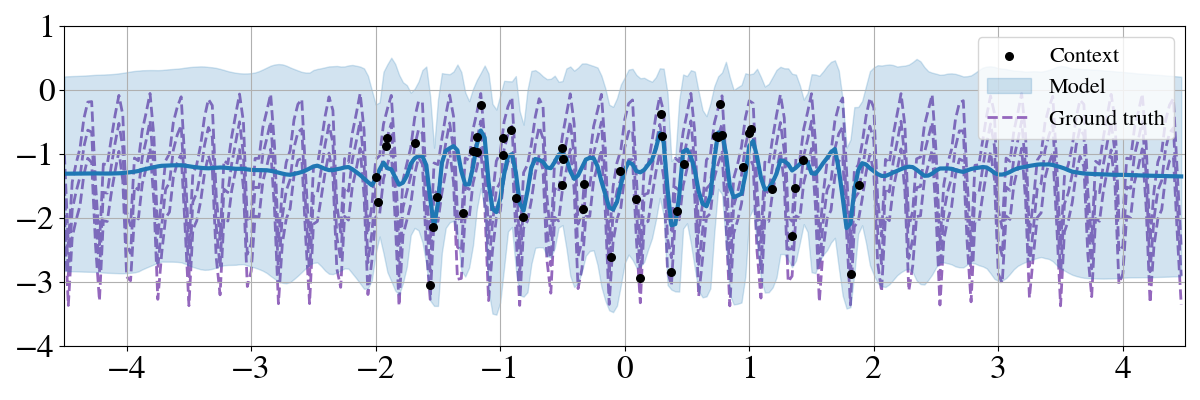}}%
        \hfill
        \subfigure[\gls{model}, Periodic $\ell = 0.24$][c]{\label{fig:gp-icicltnp-periodic-OOD}%
          \includegraphics[width=0.48\linewidth,trim={0.5cm 0.5cm 0.5cm 0.5cm},clip]{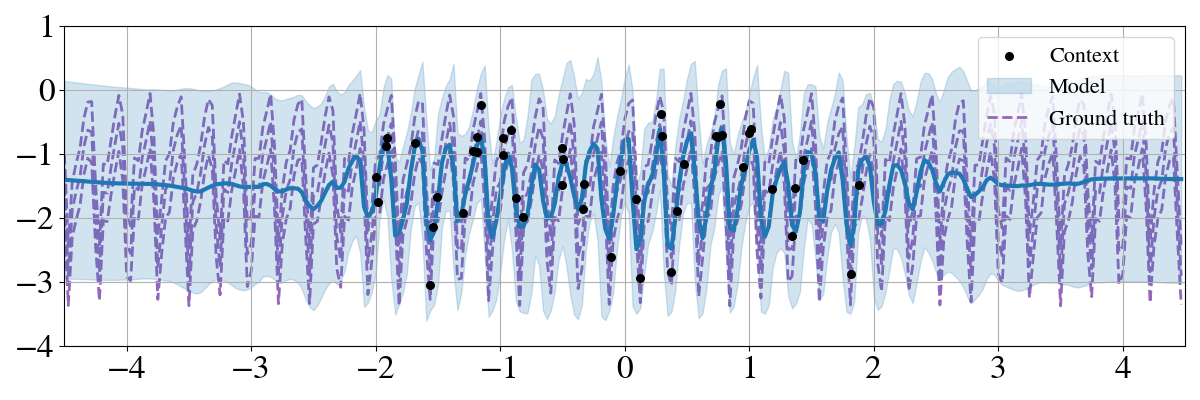}}%
      }
      {\caption{A comparison between the predictive distribution when tested OOD of the \gls{pt-tnp} (left) and the \gls{model} (right). In the top row the samples come from a GP with RBF kernel and $\ell = 5.19$. In the bottom row the samples come from a GP with periodic kernel and $\ell = 0.24$. The \gls{model} is conditioned on three in-context datasets.}
      \label{fig:gp-prediction-OOD}}
\end{figure}

\subsection{Image Completion}
\label{app:image_completion}
For each dataset and in-context datasets, we first sample a label $\ell \sim \mcU\{0, 9\}$, and then sample images from the subset of MNIST images that correspond to that label. The number of context points is sampled as $N_c \sim \mcU\{N/100, N/5\}$, the number of in-context datasets is sampled as $N_{ic} \sim \mcU\{0, 3\}$, and the number of datapoints for each of the in-context datasets is sampled as $N_{ic, j} \sim \mcU\{N/100, N/2\}$. The number of target points is set to the remaining pixels in the image (i.e.\ $N_t = N - N_c$).

For the \gls{model}, we use $M = 64$ context pseudo-tokens and $M_{ic} = 64$ in-context pseudo-tokens for each in-context dataset. For the \gls{pt-tnp}, we use $M = 64$ pseudo-tokens. Each model is trained for 500 epochs, with each epoch consisting of 625 iterations with a batch size of 16. We evaluate the performance of each model on 2,500 images (the test dataset split into groups of four, so that each group has one context dataset and three in-context datasets, so that a maximum of three is available for use).

\subsubsection{Additional Results}
\label{subsubsec:mnist-additional-results}
In \Cref{fig:ic-mnist-additional-comparison} we compare the predictive distribution of the \gls{model} when conditioning on in-context datasets drawn from different stochastic processes (MNIST labels). 

\begin{figure}[htbp]
\centering
    \floatconts
      {fig:ic-mnist-additional-comparison}
      % {\caption{A comparison between the predictive distribution of the \gls{model} and the regular \gls{tnp} when conditioning on the context distribution and in-context dataset shown. Observe that when conditioning on an in-context dataset from the same stochastic process as the context dataset (i.e.\ MNIST label 6), the \gls{model} is more confident in its predictive distribution. When the \gls{model} conditions on an in-context dataset from a different stochastic process (i.e.\ MNIST label 1), the predictive distribution is less confident.}}
      {\caption{A comparison between the predictive distribution of the \gls{model} when conditioning on in-context datasets sampled from different MNIST labels. Observe that when the context set is not compatible with the label of the in-context dataset (i.e.\ 4 and 2), the predictive distribution is less confident.}}
      {%
        \subfigure[In-context dataset label: 5][c]{\label{fig:iclbanp-ic-mnist-new-007}%
          \includegraphics[width=0.9\linewidth,trim={3cm 0.5cm 3cm 0.5cm},clip]{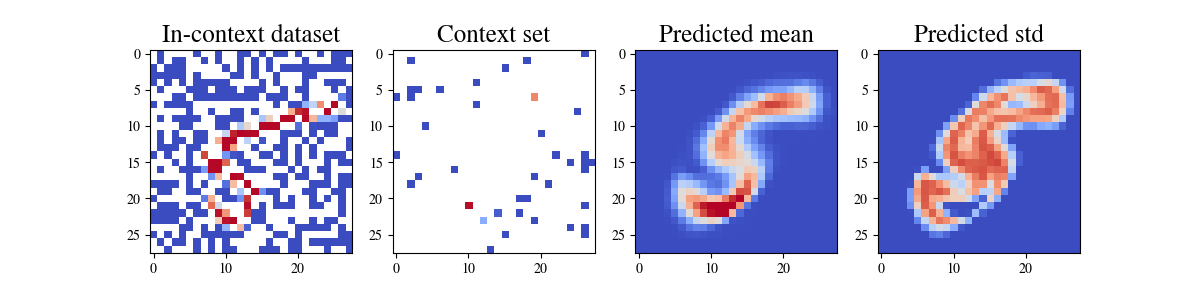}}%
        \\
        \subfigure[In-context dataset label: 1][c]{\label{fig:iclbanp-ic-mnist-new-002}%
          \includegraphics[width=0.9\linewidth,trim={3cm 0.5cm 3cm 0.5cm},clip]{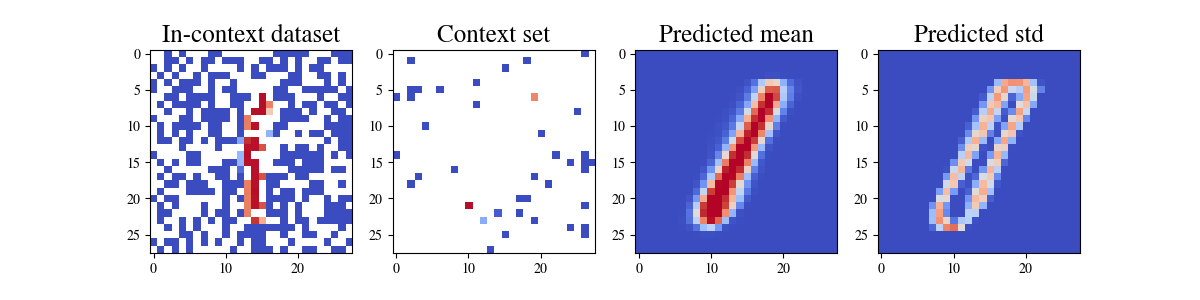}}%
        \\
        \subfigure[In-context dataset label: 4][c]{\label{fig:iclbanp-ic-mnist-new-001}%
          \includegraphics[width=0.9\linewidth,trim={3cm 0.5cm 3cm 0.5cm},clip]{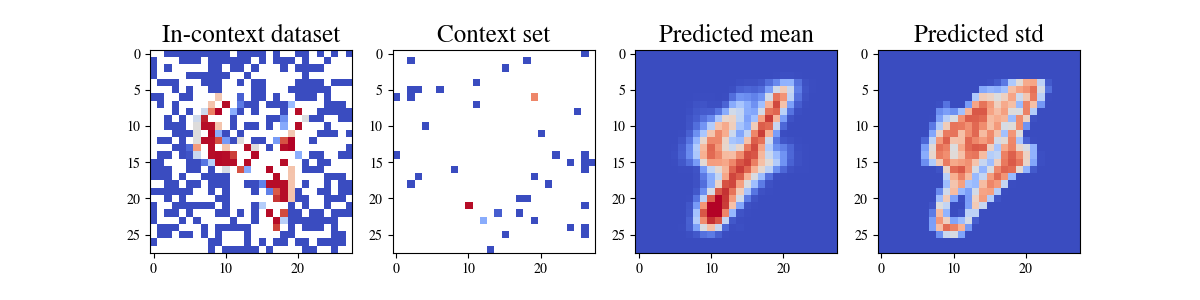}}%
        \\
        \subfigure[In-context dataset label: 2][c]{\label{fig:iclbanp-ic-mnist-new-004}%
          \includegraphics[width=0.9\linewidth,trim={3cm 0.5cm 3cm 0.5cm},clip]{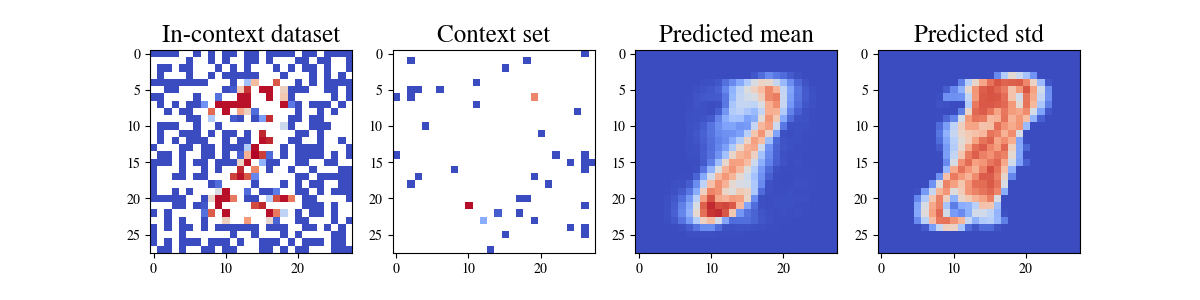}}%
      }
\end{figure}

\subsection{Environmental Data}
\label{subapp:environmental-data}
The environmental dataset consists of surface air temperatures derived from the fifth generation of the European Centre for Medium-Range Weather Forecasts (ECMWF) atmospheric reanalyses (ERA5) \citep{cccs2020}. The data has a latitudinal and longitudinal resolution of $0.5^\circ$, and temporal resolution of an hour. We consider data collected in 2019, sub-sampled at a temporal resolution of six hours. The dataset consists of data within the latitude / longitude range of $[42^{\circ},\ 53^{\circ}]$ / $[8^{\circ},\ 28^{\circ}]$ (roughly corresponding to central Europe), with the training data corresponding to the first six months of 2019, and the test data corresponding to the last six month. Individual datasets are obtained by sub-sampling the larger region, with each dataset consists of a $[10, 10, 3]$ grid spanning $5^\circ$ across each axis and 18 hours. We also provide surface elevation as additional inputs, such that $D_x = 4$. The inputs and outputs are standardised using the mean and standard deviation values obtained from data within the training region. Each dataset consists of a maximum of $N = 300$ datapoints, from which the number of context points are sampled according to $N_c \sim \mcU\{N/100, N/3\}$, with the remaining set as target points. The number of in-context datasets is sampled as $N_{ic} \sim \mcU\{0, 2\}$, and the number of datapoints in each of the in-context datasets is sampled as $N_{ic, j} \sim \mcU\{N/3, N\}$. 

For the \gls{model} and \gls{pt-tnp}, we use $M=32$ context pseudo-tokens and $M_{ic} = 32$ in-context pseudo-tokens for each in-context dataset. We train each model for 500 epochs, with each epoch consisting of 1,000 iterations with a batch size of 16.

\end{document}